\documentclass{article}

\usepackage{microtype}
\usepackage{graphicx}
\usepackage{subfig}
\usepackage{colortbl}
\usepackage{multirow}
\usepackage{listings}
\usepackage{tikz}
\usepackage{comment}
\renewenvironment{comment}{}{}
\PassOptionsToPackage{table,xcdraw}{xcolor}

\usepackage{booktabs} 
\usepackage{hyperref}

\usepackage[accepted]{icml2025}

\usepackage{amsmath}
\usepackage{amssymb}
\usepackage{mathtools}
\usepackage{amsthm}
\usepackage{enumitem,pifont}

\usepackage[capitalize,noabbrev]{cleveref}

\theoremstyle{plain}
\newtheorem{theorem}{Theorem}[section]

\newtheorem{lemma}{Lemma}

\theoremstyle{definition}
\newtheorem{definition}{Definition}

\theoremstyle{remark}

\theoremstyle{prob}
\newtheorem{prob}{Problem}
\usepackage[textsize=tiny]{todonotes}
\definecolor{shade1}{RGB}{0,218,107}
\definecolor{shade2}{RGB}{150,255,150}
\definecolor{shade3}{RGB}{220,255,200}
\definecolor{yellow}{RGB}{255,255,204}
\setlist[itemize]{leftmargin=*}

\usepackage{xspace}
\newcommand{\methodname}{\textsc{Grail}\xspace}

\newcommand{\CG}{\mathcal{G}\xspace}

\newcommand{\CB}{\mathcal{B}\xspace}
\newcommand{\CD}{\mathcal{D}\xspace}

\newcommand{\CV}{\mathcal{V}\xspace}
\newcommand{\CU}{\mathcal{U}\xspace}
\newcommand{\CE}{\mathcal{E}\xspace}

\newcommand{\CW}{\mathcal{W}\xspace}

\newcommand{\CS}{\mathcal{S}\xspace}

\newcommand{\CA}{\mathcal{A}\xspace}
\newcommand{\CJ}{\mathcal{J}\xspace}
\newcommand{\cW}{\mathbf{W}\xspace}
\newcommand{\CL}{\mathcal{L}\xspace}

\newcommand{\name}{\textsc{Grail}\xspace}
\newcommand{\gmix}{\textsc{Grail-Mix}\xspace}

\newcommand{\branch}{\textsc{Branch}\xspace}

\newcommand{\greed}{\textsc{Greed}\xspace}
\newcommand{\gedgnn}{\textsc{GedGnn}\xspace}
\newcommand{\simgnn}{\textsc{SimGnn}\xspace}
\newcommand{\graphsim}{\textsc{GraphSim}\xspace}
\newcommand{\tagsim}
{\textsc{TaGSim}\xspace}
\newcommand{\hmn}{\textsc{H2mn}\xspace}

\newcommand{\graphotsim}{\textsc{GraphOtSim}\xspace}
\newcommand{\gmn}{\textsc{Gmn}\xspace}
\newcommand{\graphedx}{\textsc{GraphEdX}\xspace}
\newcommand{\eric}{\textsc{Eric}\xspace}

\newcommand{\genn}{\textsc{Genn-A*}\xspace}
\newcommand{\ged}{\textsc{Ged}\xspace}

\newcommand*\colourcheck[1]{%
  \expandafter\newcommand\csname #1check\endcsname{\textcolor{#1}{\ding{52}}}%
}
\colourcheck{blue}
\colourcheck{green}
\colourcheck{red}
\colourcheck{ForestGreen}
\newcommand{\redxmark}{\textcolor{red}{\text{\sffamily X}}}
\newcommand{\graycircle}{\textcolor{gray}{\text{\sffamily O}}}

\begin{document}

\twocolumn[
\icmltitle{\methodname: Graph Edit Distance and Node Alignment using LLM-Generated Code}



\icmlsetsymbol{equal}{*}

\begin{icmlauthorlist}
\icmlauthor{Samidha Verma}{sch}{*}
\icmlauthor{Arushi Goyal}{yyy}{*}
\icmlauthor{Ananya Mathur}{yyy}{*}
\icmlauthor{Ankit Anand}{comp}
\icmlauthor{Sayan Ranu}{yyy}
\end{icmlauthorlist}

\icmlaffiliation{yyy}{Department of Computer Science and Engineering, IIT Delhi, India}
\icmlaffiliation{comp}{Google DeepMind, Montreal, Canada}
\icmlaffiliation{sch}{Yardi School of Artificial Intelligence, IIT Delhi, India}
\icmlcorrespondingauthor{Ananya Mathur}{cs5200416@iitd.ac.in}
\icmlcorrespondingauthor{Arushi Goyal}{cs5200418@iitd.ac.in}
\icmlcorrespondingauthor{Samidha Verma}{	samidha.verma@scai.iitd.ac.in}
\icmlcorrespondingauthor{Ankit Anand}{anandank@google.com}
\icmlcorrespondingauthor{Sayan Ranu}{sayanranu@iitd.ac.in}
\icmlkeywords{Machine Learning, ICML}

\vskip 0.3in
]



\printAffiliationsAndNotice{\icmlEqualContribution} 

\begin{abstract}
\vspace{-0.05in}
Graph Edit Distance (\ged) is a widely used metric for measuring similarity between two graphs. Computing the optimal \ged is NP-hard, leading to the development of various neural and non-neural heuristics. While neural methods have achieved improved approximation quality compared to non-neural approaches, they face significant challenges: \textit{(1)} They require large amounts of ground truth data, which is itself NP-hard to compute. \textit{(2)} They operate as black boxes, offering limited interpretability. \textit{(3)} They lack cross-domain generalization, necessitating expensive retraining for each new dataset. We address these limitations with \methodname, introducing a paradigm shift in this domain. Instead of training a neural model to predict \ged, \methodname employs a novel combination of large language models (LLMs) and automated prompt tuning to generate a \textit{program} that is used to compute \ged. This shift from predicting \ged to generating programs imparts various advantages, including end-to-end interpretability and an autonomous self-evolutionary learning mechanism without ground-truth supervision. Extensive experiments on seven datasets confirm that \methodname not only surpasses state-of-the-art \ged approximation methods in prediction quality but also achieves robust cross-domain generalization across diverse graph distributions.

\end{abstract}

\vspace{-0.3in}
\section{Introduction and Related Work}
\label{sec: intro}
\vspace{-0.1in}
\textit{Graph Edit Distance} (\ged) quantifies the dissimilarity between two graphs as the minimum number of \textit{edits} required to transform one graph into another. An edit may comprise adding or deleting nodes and edges or replacing node and edge labels. Fig.~\ref{fig:ged_example} presents an example. Computing \ged is NP-hard~\cite{ranjan2022greed} and APX-hard~\cite{grampa}. Owing to its numerous applications~\cite{gedthesis,divquery,fugal}, polynomial-time heuristics are designed in practice.

\begin{figure}[t]
    \centering
    \includegraphics[width=\linewidth]{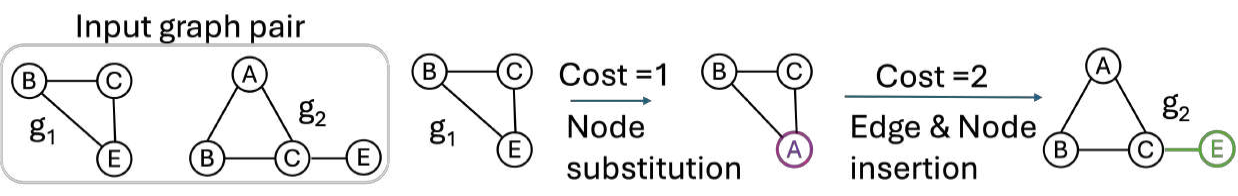}
    \vspace{-0.3in}
    \caption{Illustration of edit path from  $g_1$ to $g_2$ with \ged $3$.}
    \label{fig:ged_example}
    \vspace{-0.2in}
\end{figure}

\begin{table*}[t]
\vspace{-0.05in}
    \centering
    \scalebox{0.79}{
    \begin{tabular}{lp{1.39in}p{0.8in}p{1.299in}c}
    \toprule
       Name  & End-to-end interpretable & Cross-domain generalization & Non-reliant on NP-hard supervision & Accurate \\
       \midrule  \greed~\cite{ranjan2022greed} & \redxmark & \redxmark & \redxmark &\greencheck\\
             \gedgnn~\cite{piao2023computing} & \graycircle  & \redxmark & \redxmark &\greencheck\\
            \hmn~\cite{h2mn} & \redxmark & \redxmark & \redxmark &\greencheck\\
             \eric~\cite{eric} & \redxmark & \redxmark & \redxmark &\greencheck\\
             \graphedx~\cite{graphedx} & \redxmark & \redxmark & \redxmark &\greencheck\\
             \graphotsim~\cite{graphotsim} & \redxmark & \redxmark & \redxmark &\greencheck\\
             \graphsim~\cite{graphsim} & \redxmark & \redxmark & \redxmark &\greencheck\\
             \tagsim~\cite{tagsim} & \redxmark & \redxmark & \redxmark &\greencheck\\
             \gmn~\cite{icmlged} & \redxmark & \redxmark & \redxmark &\greencheck\\
            \genn~\cite{genn} & \graycircle & \redxmark & \redxmark &\greencheck\\
            \simgnn~\cite{simgnn} & \redxmark & \redxmark & \redxmark &\greencheck\\
            Non-neural approaches~\cite{blumenthal2020comparing} & \graycircle & \greencheck & \greencheck &\redxmark\\
         \methodname & \greencheck & \greencheck & \greencheck &\greencheck\\
         \bottomrule
    \end{tabular}}
    \vspace{-0.1in}
    \caption{Summary of the drawbacks of existing algorithms and the proposed algorithm \methodname. \greencheck{} indicates satisfaction of a desirable property, \redxmark{} indicates non-satisfaction, and \graycircle{} indicates partial satisfaction.
While \gedgnn{}, \genn{}, and traditional non-neural approaches achieve partial interpretability by providing edit paths corresponding to the \ged{}, they do not explain the semantic reasoning behind these paths. In contrast, \methodname{} achieves end-to-end interpretability through its code-based output, where each decision can be traced to its underlying logical reasoning. Non-neural approaches utilize unsupervised learning, enabling cross-domain generalization. However, their approximation errors are significantly higher on average than neural approaches, as demonstrated in \S~\ref{sec: experiments}.}
\vspace{-0.2in}
    \label{tab:gaps}
\end{table*}
\vspace{-0.1in}
\subsection{Existing Works and their Limitations}
\vspace{-0.05in}
Existing heuristics to approximate \ged can be broadly grouped into two paradigms: non-neural and neural.\\
\textbf{Non-Neural Methods:} A comprehensive survey on non-neural methods is available at~\citet{blumenthal2020comparing}. 
These approaches leverage techniques such as transformations to the linear sum assignment  (\textsc{NODE}~\cite{NODE_ADJ_IP}, \textsc{BRANCH-TIGHT}~\cite{BRANCH_TIGHT}),  mixed integer programming (\textsc{MIP}) (\textsc{LP-GED-F2}~\cite{lerouge2017new}, \textsc{ADJ-IP}~\cite{NODE_ADJ_IP}, \textsc{COMPACT-MIP}~\cite{COMPACT_MIP}), local search methods (\textsc{IPFP}~\cite{leordeanu2009integer}), and approximations to the quadratic asssignment problem~\cite{eugene}.\\ 
\textbf{Neural Methods: } Recent literature shows a shift towards graph neural network-based approaches for approximating \ged, driven by their superior approximation quality compared to non-neural methods~\cite{ranjan2022greed,h2mn,simgnn,piao2023computing,genn,eric,graphedx,graphsim,graphotsim,icmlged}. However, these advancements come with limitations, as summarized in Table~\ref{tab:gaps}. 
\vspace{-0.3in}
\begin{itemize}
  \item \textbf{Lack of interpretability: }Most neural methods only predict the \ged and not the corresponding edit path. The edit path is essential for various applications such as identifying functions of protein complexes ~\cite{doi:10.1073/pnas.0806627105}, image alignment ~\cite{Conte2003GraphMA}, and uncovering gene-drug regulatory pathways ~\cite{chen2019hogmmnc}. Few neural methods that predict the edit path~\cite{piao2023computing,genn} rely on expensive ground truth computation, which can only be attained for very small graphs ($\approx 10$ nodes). For larger graphs, random edits are made to synthetic graphs to generate the training samples. 
 \vspace{-0.1in} 
    \item \textbf{NP-hard training data: }The training dataset for neural methods consists of graph pairs and their true \ged. \ged computation is NP-hard. Therefore, generating this training data is prohibitively expensive and restricted to small graphs only. Hence, approximation error deteriorates on larger graphs.~\cite{ranjan2022greed}
   \vspace{-0.05in} 
    \item \textbf{Lack of generalization: }Neural \ged approximators struggle to generalize across datasets. For datasets from different domains (e.g., chemical compounds vs. function-call graphs), the node label sets often differ. Since the number of parameters in a GNN depends on the feature dimensions of the nodes, GNNs fail to generalize across domains. Even within the same domain, as demonstrated later in~\S\ref{sec: experiments}, distribution shifts in structural and node label distributions lead to increased approximation error. This limitation necessitates generating ground-truth data and training separate models for each dataset. Given that generating training data is NP-hard, this pipeline becomes highly resource-intensive.
\end{itemize}
\vspace{-0.2in}
\subsection{Contributions} 
\vspace{-0.05in}
We address the above-outlined limitations through \methodname: \underline{Gr}aph Edit Dist\underline{a}nce and Node Al\underline{i}gnment using \underline{L}LM-Generated Code. \methodname introduces a paradigm shift in the domain of \ged approximations through the following novel innovations. 


\vspace{-0.15in}
\begin{itemize}
\item {\bf Problem formulation:} We shift the learning objective from approximating \ged to learning a \textit{program} that approximates \ged. This reformulation provides end-to-end interpretability, as each algorithmic decision can be traced to its underlying logical reasoning. Moreover, by elevating the output to a higher level of abstraction through code generation, we achieve superior generalization across datasets, domains, graph sizes, and label distributions. 
\vspace{-0.2in}
\item {\bf LLM-guided program discovery: } The algorithmic framework of \methodname is grounded on three novel design choices. First, we map the problem of approximating \ged to \textit{maximum weight bipartite matching}, where the weights of the bipartite graph are computed using an LLM-generated program. Second, the prompt provided to the LLM is tuned through an evolutionary algorithm~\cite{romera2024mathematical}. Third, our prompt-tuning methodology eliminates the need for ground-truth \ged data by designing a prediction framework where the prediction is guaranteed to be an \textit{upper bound} to the true \ged. Hence, minimizing the upper-bound is equivalent to minimizing the approximation error, thereby overcoming a critical bottleneck of existing neural approaches.
\vspace{-0.07in}
\item \textbf{Comprehensive Empirical Evaluation:} Through extensive experiments across 6 datasets, we demonstrate that \methodname discovers \textit{foundational} code-based heuristics. Specifically, these heuristics not only surpass the state-of-the-art methods in \ged computation but also exhibit generalization across diverse datasets and domains. This crucial feature eliminates the need for costly dataset-specific training, thereby addressing a significant limitation of existing neural algorithms.
\end{itemize}

\vspace{-0.25in}
\section{Preliminaries and Problem Formulation}
\label{sec:formulation}
\vspace{-0.05in}
\begin{definition}[Graph] \textit{We represent a node-labeled undirected graph as $\mathcal{G(V,E, L)}$ where $\mathcal{V}$ = $\{v_1, \cdots, v_{|\mathcal{V}|}\}$ is the set of nodes, $\mathcal{E}\subseteq \CV\times\CV$ is the edge set and $\mathcal{L}$ : $\CV \rightarrow \Sigma$ is a labeling function that
maps nodes to labels, where $\Sigma$ is the set of all labels.}
\end{definition}
\vspace{-0.05in}
In unlabeled graphs, all nodes are assigned the same label.
\begin{definition}[Node Mapping] \textit{A node mapping between two graphs $\mathcal{G}_1$ and $\mathcal{G}_2$, each consisting of $n$ nodes, refers to a bijection $\pi$ : $\mathcal{V}_1$ $\rightarrow$ $\mathcal{V}_2$, where every node $v$ $\in$ $\mathcal{V}_1$ is uniquely mapped to a node $\pi(v)$ $\in$ $\mathcal{V}_2$. }
\end{definition}
\vspace{-0.1in}
\textbf{Extension to graphs of different sizes:} When dealing with two graphs $\mathcal{G}_1$ and $\mathcal{G}_2$ with different numbers of nodes, $n_1$ and $n_2$ respectively, such that $n_1 < n_2$, the smaller graph $\mathcal{G}_1$ can be extended to match the size of $\mathcal{G}_2$ by introducing $n_2-n_1$ additional isolated \textit{dummy} nodes. These new nodes are labeled with a unique identifier, $\epsilon$, indicating that they are placeholders with no connections. From this point onward, we assume that any pair of graphs in consideration have an equal number of nodes, with smaller graphs being augmented by dummy nodes as necessary.
\begin{definition}[\ged under a node mapping $\pi$] \textit{Given a node mapping $\pi$, the cost function for calculating graph edit distance between graphs $\mathcal{G}_1(\CV_1,\CE_1,\CL_1)$ and $\mathcal{G}_2(\CV_2,\CE_2,\CL_2)$ is expressed as: 
\vspace{-0.1in}
\[
\text{\ged}_{\pi}(\mathcal{G}_1, \mathcal{G}_2) = \sum_{v_1 \in \mathcal{V}_1} \mathbb{I}\left(\CL_1(v_1) \neq \CL_2(\pi(v_1))\right)
\]
\vspace{-0.15in}
\[
+ \frac{1}{2}\sum_{v_1 \in \mathcal{V}_1}\sum_{v_2 \in \mathcal{V}_1}\mathbb{I}\left( e_1(v_1, v_2)\neq e_2(\pi(v_1), \pi(v_2)) \right)
\]
\vspace{-0.1in}
where,
\vspace{-0.1in}
\begin{itemize}
    \item $e_i(u,v)$ returns 1 if the edge $(u,v)\in\CE_i$ in graph $\CG_i$, $0$ otherwise.
    \vspace{-0.1in}
    \item $\mathbb{I}(A)$ is the indicator function, which is $1$ if the condition $A$ holds, and $0$ otherwise.
\end{itemize}}
\end{definition}
\vspace{-0.1in}
\textbf{Interpretation of edit path from node mapping:} The first part of the equation captures node mismatches where it evaluates the label differences between nodes in $\mathcal{G}_1$ and $\mathcal{G}_2$. 
Mapping a dummy node to a real node (or vice versa) results in a label mismatch, reflecting the insertion or deletion of a node, while a mismatch between real nodes denotes a substitution. The second part of the equation captures edge mismatches. Specifically, if an existing edge in $\mathcal{G}_1$ (i.e., $e_1
(v_1, v_2) = 1$) is mapped to a non-existing edge in $\mathcal{G}_2$ (i.e., $e_2(\pi(v_1), \pi(v_2)) = 0$) or vice versa, the cost is $1$ representing edge deletion and insertion, respectively. 
\begin{definition}[Graph edit distance (\ged)]
\label{def:ged}
\textit{The \ged between graphs $\CG_1$ and $\CG_2$ is the minimum \ged across all possible node mappings.}
\vspace{-0.05in}
\begin{equation}
    \text{\ged}(\CG_1,\CG_2)=\min_{\forall \pi\in\mathcal{M}}\left\{\ged_{\pi}(\CG_1,\CG_2)\right\}
\end{equation}
\vspace{-0.2in}

\textit{Here, $\mathcal{M}$ denotes the universe of all possible mappings.}
\end{definition}
\vspace{-0.1in}
The problem is hard (NP-hard and APX-hard) since the cardinality of $\mathcal{M}$ is $n!$, where $n=\max\{|\CV_1|,|\CV_2|\}$.

The problem of learning to code for approximating \ged is defined as follows.
\begin{prob}[Learning to code for \ged] 
\label{prob:ged}
Given a set of training graph pairs $\mathbb{T}=\{\langle\mathcal{G}_1, \mathcal{G}'_1\rangle,\langle\mathcal{G}_2, \mathcal{G}'_2\rangle,\cdots,\langle\mathcal{G}_n, \mathcal{G}'_n\rangle\}$, learn a program $P:(\CG_t,\CG'_t)\rightarrow \mathbb{Z}+$ that takes as input a graph pair $\langle\mathcal{G}_t,\mathcal{G}'_t\rangle\in\mathbb{T}$, and outputs a non-negative integral distance that minimizes
\vspace{-0.1in}
\begin{equation}
\label{eq:prob1}
\sum_{t=1}^n\left\lvert P(\CG_t,\CG'_t)-\ged(\mathcal{G}_t, \mathcal{G}'_t)\right\rvert
\end{equation}
\end{prob}
\vspace{-0.05in}
Note that our training set consists solely of graph pairs, without requiring their true \ged, which is computationally prohibitive due to its NP-hardness. As we will elaborate in the next section, we identify polynomial-time computable upper bounds for the true \ged and reformulate the optimization objective to minimize this upper bound. This autonomous self-evolutionary learning mechanism overcomes a significant limitation of neural \ged approximators. 
\begin{figure*}
    \centering
    \includegraphics[width=5.9in]{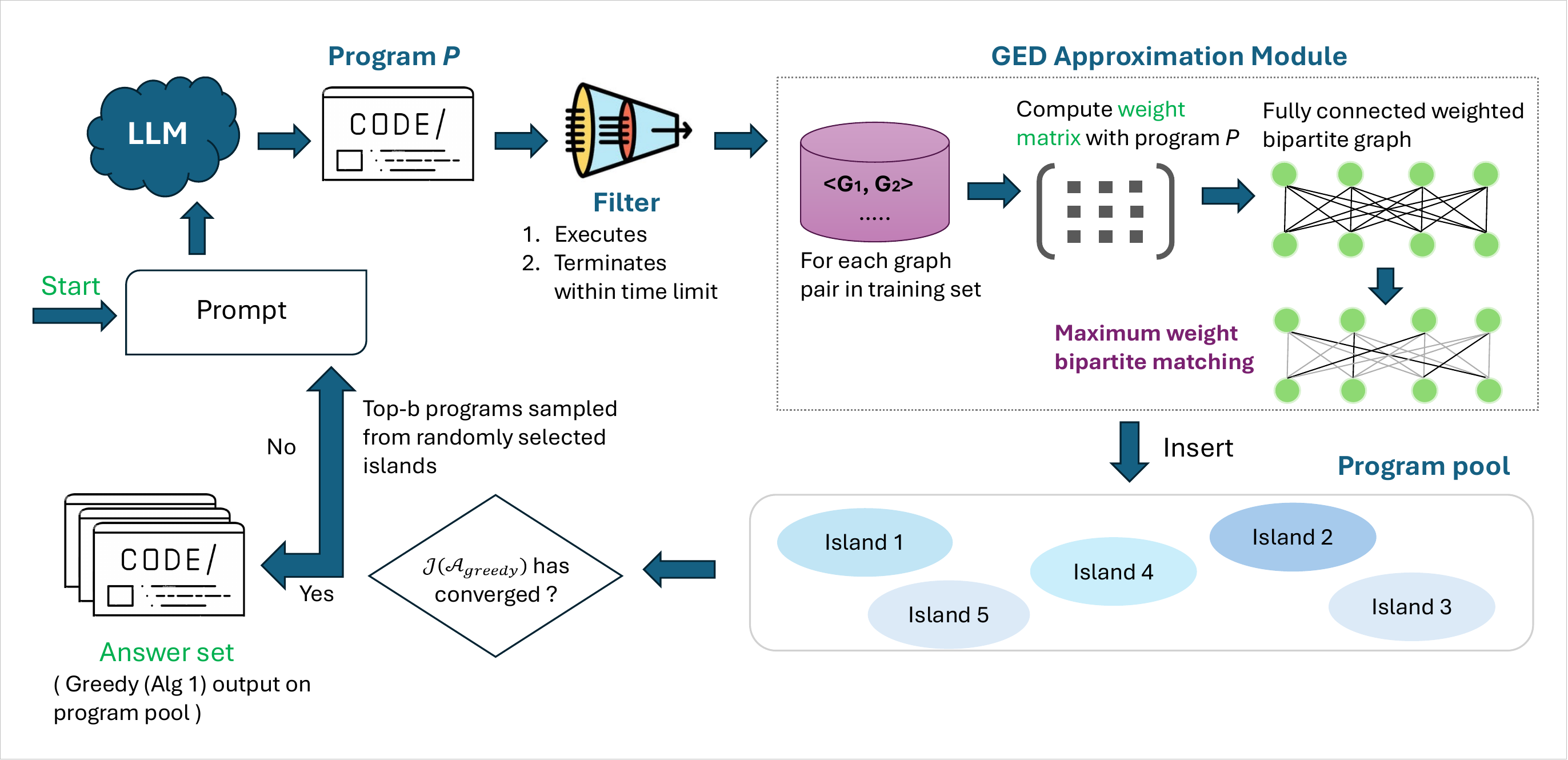}
    \vspace{-0.35in}
    \caption{Pipeline of \methodname.}
    \label{fig:pipeline}
    \vspace{-0.2in}
\end{figure*}
\vspace{-0.1in}
\section{Approximation Strategy}
\label{sec:approx}
\vspace{-0.05in}
The true \ged corresponds to the minimum distance across all possible node mappings (Def.~\ref{def:ged}). However, enumerating all such mappings is computationally infeasible due to its factorial complexity relative to graph size. To overcome this challenge, we approximate the \ged by evaluating a small subset of mappings (e.g., $15$) and selecting the minimum distance among them. These mappings are generated by programs derived from the LLM, as detailed in \S~\ref{sec: method}. Importantly, this approximated \ged serves as an upper bound to the true \ged, as it considers only a subset of all possible node mappings.
\vspace{-0.1in}
\subsection{Node Mappings through Bipartite Matching}
\label{sec:taskspecification}
\vspace{-0.05in}
The task of mapping nodes between two graphs can be approximated as \textit{Maximum Weight Bipartite Matching}.
\vspace{-0.05in}
\begin{definition}[Maximum Weight Bipartite Matching]
    \textit{Given a weighted bipartite graph $\CB(\CV,\CU,\CE,\CW)$ with node sets $\CV$ and $\CU$, and a weighted edge set $\CE:\CV\times\CU\rightarrow\mathbb{R}$ where $\CW:\CE\rightarrow\mathbb{R}$ assigns weights to edges, find a subset of edges $\CE^*\subseteq\CE$ that \textbf{(1)} induces a bijection between nodes in $\CV$ and nodes in $\CU$, and \textbf{(2)} maximizes the total weight of the mapped edges, i.e., $\sum_{e\in\CE^*}\CW(e)$. Here, $\CW(e)$ represents the weight of edge $e$.}
\end{definition}
\vspace{-0.1in}
Maximum weight bipartite matching can be solved optimally in polynomial time using the Hungarian algorithm~\cite{kuhn1955hungarian}. Additionally, several heuristics have been proposed, such as the Hopcroft–Karp Algorithm~\cite{hopcroft1973n}, the Neighbor-biased mapper~\cite{ctree}, or greedy selection of the highest-weight edges while maintaining the one-to-one mapping constraint. We use the notation $\pi(\CB)$ to denote the node mapping obtained from $\CB$.

To use maximum weight bipartite matching for approximating \ged, for given graphs $\CG_1(\CV_1,\CE_1,\CL_1)$ and $\CG_2(\CV_2,\CE_2,\CL_2)$, we construct a fully connected, weighted bipartite graph $\CB(\CV_1,\CV_2,\CE,\CW)$ where edge set $\CE=\{(v_1,v_2) \mid v_1 \in \CV_1, v_2 \in \CV_2\}$\footnote{Recall, we assume the smaller graph is padded with dummy nodes to ensure $|\CV_1| = |\CV_2|$.}. The weight of an edge $(v_1, v_2)$ is set based on some policy, which should ideally reflect the probability of $v_1$ being mapped to $v_2$ in the optimal \ged mapping. Maximum weight bipartite matching is then performed on $\CB$ using any standard algorithm, and the \ged is computed based on the mapping $\pi(\CB)$.

The quality of the mapping with respect to approximating \ged, therefore, rests on the edge weights in the bipartite graph. We will use an LLM to learn the policy, in the form of a program, with the following \textit{minimization objective}.
\begin{prob}[Weight Matrix Generation]
\label{prob:weight}
Given train set $\mathbb{T}=\{\langle\mathcal{G}_1, \mathcal{G}'_1\rangle,\cdots,\langle\mathcal{G}_n, \mathcal{G}'_n\rangle\}$, generate a program $P$ that takes as input each pair $\langle\mathcal{G}_t,\mathcal{G}'_t\rangle\in\mathbb{T}$ and outputs a corresponding weight matrix $\cW_t \in \mathbb{R}^{|\CV_t| \times |\CV'_t|}$ minimizing
\vspace{-0.15in}
\begin{equation}
\label{eq:gedp}
  \sum_{t=1}^n  \ged_{\pi(P)}(\CG_t, \CG'_t)
\end{equation}

\vspace{-0.19in}
Here, $\cW_t[i,j]$ denotes the weight of edge $(v_i,v_j)$ where $v_i\in\CV_t$ to node $v_j\in\CV'_t$. $\pi(P)$ denotes the mapping generated by maximum weight bipartite matching on the bipartite graph formed by $P$.
\end{prob}
\vspace{-0.1in}
\subsection{Budget-constrained Selection of Node Maps}
\label{sec:score}
\vspace{-0.05in}
Let $\CD=\{P_1,\cdots,P_m\}$ be the set of programs generated by the LLM and $\pi(P_i)$ denote the mapping produced by program $P_i\in\CD$. From Def.~\ref{def:ged}, we know $\ged_{\pi(P_i)}(\CG_1,\CG_2)\geq\ged(\CG_1,\CG_2)$, i.e., each program provides an \textit{upper bound} on the true \ged. The smaller the upper bound, the closer we are to the true \ged. Our goal is to select a subset $\CA^*\subseteq\CD$  of $b$ programs that minimize the cumulative upper bounds across all train graph pairs. $b\ll |\CD|$ denotes the maximum number of mappings we are allowed to evaluate. These $b$ programs will finally be used during inference for unseen graph pairs. Formally, this presents us with the following optimization problem.
\begin{prob}[Map Selection]
\label{prob:map}
\textit{Given programs $\CD=\{P_1,\cdots,P_m\}$, select $\CA^*$ such that:}
\vspace{-0.05in}
\begin{alignat}{2}
\label{eq:minA}
\CA^*&=\arg\min_{\forall \CA\subseteq\CD,\:|\CA|=b}\left\{\mathcal{J}(A)\right\}\\
\label{eq:J}
\mathcal{J}(\CA)&=\sum_{\langle\CG_1,\CG_2\rangle\in\mathbb{T}}\min_{P\in\CA}\left\{
\ged_{\pi(P)}(\CG_1,\CG_2)\right\}
\end{alignat}
$\mathcal{J}(\CA)$ \textit{quantifies the quality of the subset of mappings in $\CA$.}
\end{prob}
\begin{theorem}
Prob.~\ref{prob:map} is NP-hard.
\end{theorem}
\vspace{-0.15in}
\begin{proof} The proposed optimization problem reduces to the \textit{Set Cover} problem~\cite{cormen}, rendering it NP-hard. For the formal proof, please refer to App.~\ref{app:setcover}.
\end{proof}
\vspace{-0.1in}

Owing to NP-hardness, finding the optimal subset of mappings $\CA^*$ is not feasible in polynomial time. We establish that $\CJ(\CA)$ is \textit{monotonic} and \textit{submodular} (refer App.~\ref{app:submodularity}). This enables us to use the greedy hill-climbing algorithm (Alg.~\ref{alg:greedy}) to select a sub-optimal but reasonable subset of programs, $\CA_{greedy}$.\\

\begin{algorithm}[b]
{\footnotesize
    \caption{The greedy approach}
    \label{alg:greedy}
    \begin{algorithmic}[1]
    \REQUIRE Train data $\mathbb{T}=\{T_1,\cdots,T_n\}$ where  $T_t=\langle\CG_t,\CG'_t\rangle$ is a pair of graphs, budget $b$.
    \ENSURE solution set $\CA_{greedy}$, $|\CA_{greedy}|=b$
        \STATE $\CA_{greedy} \leftarrow \emptyset$
        \WHILE{$size(\CA_{greedy})\leq b$  (within budget)}
            \STATE $P^*\leftarrow \arg\max_{P\in \CD\backslash \CA_{greedy}}\{\CJ(\CA_{greedy}\cup\{P\})-\CJ(\CA_{greedy})\}$
            \STATE $\CA_{greedy}\leftarrow \CA_{greedy}\cup \{P^*\}$ 
        \ENDWHILE
    \STATE \textbf{Return} $\CA_{greedy}$
    \end{algorithmic}}
\end{algorithm}
\vspace{-0.25in}
\section{\methodname: Proposed Methodology}
\label{sec: method}
\vspace{-0.05in}
In \S~\ref{sec:approx}, we decompose Prob.~\ref{prob:ged} into two subproblems: weight selection in a bipartite graph (Prob.~\ref{prob:weight}) and budget-constrained map selection (Prob.~\ref{prob:map}). Prob.~\ref{prob:map} is solved (approximately) using Alg.~\ref{alg:greedy}. Hence, to complete our approximation scheme, we need to solve Prob.~\ref{prob:weight}.

Fig.~\ref{fig:pipeline} presents the pipeline of \methodname. The process begins with an initial prompt that specifies a trivial program for weight selection, and the LLM is tasked with improving this program for \ged computation via bipartite matching (details in \S~\ref{sec:prompt}). Each newly generated program is verified for syntactic correctness and must terminate within a predefined time limit. If these criteria are met, the program is evaluated on the training set of graph pairs and added to the program pool along with its \textit{score}, which reflects its marginal contribution to $\CJ(\CA_{greedy})$. A new prompt is then constructed by sampling the highest-scoring programs from the current pool. The LLM refines these programs, generating new candidates to further enhance performance. These newly generated programs are evaluated and added to the pool following the same procedure. This iterative process continues until $\CJ(\CA_{greedy})$ converges, ensuring that improvements stabilize across iterations. The following sections detail each step of this process.

\vspace{-0.1in}
\subsection{Prompt Specification}
\label{sec:prompt}
\vspace{-0.05in}
The prompt is a computer program consisting of three distinct components: \textbf{(1)} the problem description, \textbf{(2)} the task specification, and \textbf{(3)} the top-$k$ programs generated so far based on a scoring function. A sample prompt is provided in Fig.~\ref{fig:prompt} in the Appendix. Here, $k$ is set to $2$.\looseness=-1

\textbf{Problem Description:} The problem description includes the definition of \ged, which is embedded as a comment within the program (refer to Fig.~\ref{fig:prompt}).

\textbf{Task Specification:} The LLM's task is defined through a comment specifying the inputs it should expect and the required output. The output is a weight matrix $\cW \in \mathbb{R}^{|\CV_1| \times |\CV_2|}$ for the bipartite graph, where $\cW[i,j]$ quantifies the strength of mapping node $v_i \in \CV_1$ to node $v_j \in \CV_2$ in the context of \ged computation. The input includes the graph pair represented by their adjacency matrices and an initial weight matrix $\cW^0 \in \mathbb{R}^{|\CV_1| \times |\CV_2|}$ with the same dimensions and semantics as the output. During execution, the input weight matrix $\cW^0$ is initialized such that $\cW^0[i,j] = 1$ if the corresponding nodes share the same label, i.e., $\CL_1(v_i) = \CL_2(v_j)$ for $v_i \in \CV_1$ and $v_j \in \CV_2$, and $\cW^0[i,j] = 0$ otherwise. Additionally, the header of the function that the LLM needs to generate is explicitly provided.\looseness=-1

\textbf{Top-$k$ Programs:} The initial prompt includes a trivial program where $\forall v_i \in \CV_1, v_j \in \CV_2, \cW[i,j] = 0$. In subsequent iterations, $k$ high-scoring programs are sampled for inclusion in the prompt, where $k$ is a hyper-parameter. The scoring and sampling methodology are described in \S~\ref{sec:prompt_tuning}.
\vspace{-0.1in}
\subsection{Prompt Tuning}
\label{sec:prompt_tuning}
\vspace{-0.05in}
{\bf Filter:} After a program is generated, it undergoes a filtering step to verify that it executes and terminates on training graph pairs within a predefined time limit. Programs that fail this filter are discarded. For those that pass, we compute their score and add them to our program database $\CD$. 

\textbf{Score computation:} 
In Prob.~\ref{prob:map}, we take the minimum \ged across all selected mappings in the answer set. A program's utility, therefore, depends on how it complements other programs in the answer set.
 Hence, we define its score as the \textit{marginal contribution} to the objective function $\CJ(\CA_{\text{greedy}})$. Specifically, we execute Alg.~\ref{alg:greedy} on the current pool of programs, where $\CA = \{P_1, \ldots, P_i\}$ represents the subset of programs selected up to iteration $i$. If program $P$ is added in the $i+1$-th iteration due to providing the highest marginal contribution, its score is computed as:  
\vspace{-0.05in}
\begin{equation}
\text{score}(P) = \CJ(\CA \cup \{P\}) - \CJ(\CA),
\label{eq:marginal_gain_eq}
\end{equation}

\vspace{-0.15in}
{\bf Evolutionary program selection:} The next stage involves selecting programs from the pool to be included in the next prompt. We use the evolutionary algorithm proposed in Funsearch ~\cite{romera2024mathematical} for evolving our programs generated by LLM. Since the programs evolve through mutations introduced by the LLM, the selection mechanism optimizes two distinct objectives. First, the sampled programs should have high scores. Second, the sampled programs should have smaller length improving the interpretability of generated programs.

The evolutionary algorithm follows the \textit{islands model}~\cite{island}. Specifically, the population of existing programs is partitioned into $s$ islands, where $s$ is a hyperparameter. Initially, all islands are empty. When a program is added to the pool, it is randomly assigned to an island. Subsequently, to decide which $k$ programs are included in the prompt, we randomly choose an island. Similar to Funsearch, the programs within each island are then split into clusters depending on score. After selecting the island, clusters are selected based on softmax distribution on score. Within a clusters, the programs are selected based on length (smaller is better). Hence, the program selection mechanism for the next prompt favors higher scores and shorter lengths. More details of the process can be found in ~\citet{romera2024mathematical}. The LLM is then tasked with further improving these programs.

With this design, each island evolves independently. To enable cross-fertilization among islands, we periodically discard half of the islands which have the lowest score. The discarded islands are replaced by  iterating over each of the surviving islands, and selecting its best program to seed the replacement population.
\vspace{-0.1in}
\subsection{Training and Inference}
\label{sec:training}
\vspace{-0.05in}
\textbf{Training:} As illustrated in Fig.~\ref{fig:pipeline}, each iteration involves generating a program, scoring it, and assigning it to an island before constructing and executing a new prompt. The quality of the program pool is measured by $\CJ(\CA_{greedy})$, serving as an analog to a loss function in our framework. This iterative process continues until $\CJ(\CA_{greedy})$ converges, defined as its improvement over the last $i$ iterations falling below a predefined threshold, akin to the \textit{patience} parameter in neural model training. 

Overall, \methodname seeks to minimize the upper bound of \ged. With this strategy, we bypass the need for ground-truth \ged data, a key bottleneck in training neural approaches. This unique design is not feasible in neural pipelines since the prediction can err on either side of the true distance. 

\textbf{Inference:} During inference, we directly return $\CJ(\CA_{greedy})$ for the given input graph pair. Note that since the output of the training phase is executable code, inference is CPU-bound, enabling it to operate in low-resource environments.

\vspace{-0.1in}
\section{Experiments}
\label{sec: experiments}
\vspace{-0.05in}
In this section, we benchmark \methodname and establish that:
\vspace{-0.15in}
\begin{itemize}
\item \textbf{Approximation Error:} \methodname achieves low approximation errors and consistently ranks among the top algorithms across all six datasets. Notably, unlike neural approximators, it achieves this performance without relying on extensive NP-hard ground-truth \ged training data.
   \vspace{-0.23in}
\item \textbf{Foundational heuristics:} \methodname breaks new ground by generating heuristics that generalize across diverse datasets, including those featuring unseen node labels and varying graph sizes. This exceptional adaptability sets \methodname apart, as no existing neural \ged approximators have demonstrated such versatility.
\end{itemize}
\vspace{-0.15in}
The codebase of \methodname and the programs generated for the various datasets are available at \url{https://github.com/idea-iitd/Grail}.  
\begin{table}[t!]
\vspace{-0.1in}
\centering
\caption{Datasets used for benchmarking \methodname.}
\label{tab:datasets}
\setlength{\tabcolsep}{4pt} 
\renewcommand{\arraystretch}{1.1} 
\scalebox{0.8}{
\begin{tabular}{lrrrrl}
\toprule
\textbf{Name} & \textbf{\# Graphs} & \textbf{Avg $|V|$} & \textbf{Avg $|E|$} & \textbf{\# labels} & \textbf{Domain} \\ \midrule 
ogbg-molhiv   & 39650              & 24                 & 52                 & 119                & Molecules         \\ 
ogbg-molpcba  & 436313             & 26                 & 56                 & 119                & Molecules        \\ 
ogbg-code2    & 139468             & 37                 & 72                 & 97                 & Software        \\ 
AIDS          & 700                & 9                  & 9                  & 29                 & Molecules         \\ 
Linux       & 1000               & 8                  & 7                  & Unlabeled                 & Software        \\ 
IMDB          & 1500               & 13                 & 65                 & Unlabeled                 & Movies          \\ 
ogbg-ppa   & 39650              & 243.4                 & 2226.1                 & Unlabeled                & Protein         \\
\bottomrule
\end{tabular}}
\vspace{-0.2in}
\end{table}

\begin{table*}[ht!]
\vspace{-0.1in}
\centering
\caption{\textbf{RMSE Comparison:} The top-3 lowest RMSEs per dataset are highlighted in green shades, with darker shades denoting better RMSE. An asterisk (*) marks the better value when additional decimal places resolve ties after rounding to two decimal places.}
\label{tab:rmse}
\setlength{\tabcolsep}{7pt}
\scalebox{0.85}{
\begin{tabular}{@{}p{1.4 cm}l|>{\centering\arraybackslash}p{1.2cm}>{\centering\arraybackslash}p{1.2cm}>{\centering\arraybackslash}p{1.3cm}>{\centering\arraybackslash}p{1.9 cm}>{\centering\arraybackslash}p{1.8cm}>{\centering\arraybackslash}p{2.2cm}|>{\centering\arraybackslash}p{1.6cm}@{}}
\toprule
\textbf{Type} & \textbf{Methods} & \textbf{AIDS} & \textbf{Linux} & \textbf{IMDB} & \textbf{ogbg-molhiv} & \textbf{ogbg-code2} & \textbf{ogbg-molpcba} & \textbf{Avg. Rank} \\ 
\midrule
\textbf{LLM} & \methodname   & \cellcolor{shade1}0.57 & \cellcolor{shade2}0.13 & \cellcolor{shade2}0.55 & \cellcolor{shade2}2.96* & \cellcolor{shade2}4.22 & \cellcolor{shade3}3.18 & 2 \\ 
& \methodname-\textsc{Mix}   & \cellcolor{shade3}0.64 & \cellcolor{shade1}0.11 & \cellcolor{shade1}0.53 & \cellcolor{shade2}2.96 & \cellcolor{shade1}4.10 & 3.40 & 2.17 \\ 
\midrule
\multirow{4}{*}{\textbf{Neural}} 
    & \greed        & \cellcolor{shade2}0.61 & 0.41 & 4.8 & 3.02 & \cellcolor{shade3}5.52 & \cellcolor{shade1}2.48 & 3.5 \\ 
    & \gedgnn       & 0.92 & 0.29 & \cellcolor{shade3}4.43 & \cellcolor{shade1}1.75 & 16.68 & 4.58 & 5 \\ 
    & ERIC        & 1.08 & 0.30 & 42.44 & 3.56 & 17.55 & \cellcolor{shade2}2.79 & 6.5 \\ 
    & H$^2$MN      & 1.14 & 0.60 & 57.8 & 12.01 & 11.96 & 5.50 & 8.33 \\ 
    & GRAPHEDX     & 0.78 & 0.27 & 32.36 & 14.14 & 21.46 & 10.01  & 8.33 \\
\midrule
\multirow{6}{*}{\textbf{Non Neural}} 
    & \textsc{Adj-Ip}       & 0.85 & 0.50 & 42.18 & 10.21 & 14.94 & 8.06 & 7.33 \\ 
    & \textsc{Node}        & 2.71 & 1.24 & 61.03 & 4.97 & 8.34 & 4.94 & 8.17 \\ 
    & \textsc{Lp-Ged-F2}    & 1.96 & \cellcolor{shade3}0.23 & 55.26 & 12.86 & 16.03 & 10.30 & 8.83 \\ 
    & \branch & 3.31 & 2.45 & 7.36 & 9.86 & 12.64 & 11.31 & 9.33 \\
    & \textsc{Compact-Mip}  & 2.69 & 0.44 & 65.88 & 10.88 & 19.46 & 8.81 & 10 \\  
    & \textsc{Ipfp}         & 4.18 & 2.29 & 69.45 & 13.69 & 15.19 & 10.02 & 11.5 \\ 
\bottomrule
\end{tabular}}
\vspace{-0.1in}
\end{table*}
\begin{table*}[ht!]
\vspace{-0.1in}
\centering
\caption{\textbf{EMR Comparison}: The top-$3$ highest EMRs per dataset are highlighted in green shades, with darker shades denoting better EMRs. The EMR values for \methodname are in App. (Table ~\ref{tab:EMR_genie}). We omit them here as \methodname-\textsc{Mix} performs similarly across datasets. For a focused comparison, we only include the top-3 baselines from Table~\ref{tab:rmse}, since the remaining do not provide competitive performance. For ties after rounding to two decimals, an asterisk (*) marks the higher value. Values in (0.99, 1) are shown as $\approx1$.} 
\label{tab:EMR_main}
\setlength{\tabcolsep}{9pt}
\scalebox{0.85}{
\begin{tabular}{@{}p{1.8cm}|l>{\centering\arraybackslash}p{1.7cm}>{\centering\arraybackslash}p{1.7cm}>{\centering\arraybackslash}p{2.0cm}>{\centering\arraybackslash}p{2.0cm}>{\centering\arraybackslash}p{2.2cm}|>{\centering\arraybackslash}p{1.8cm}>{\centering\arraybackslash}p{1.4cm}@{}}
\toprule
 \textbf{Methods} & \textbf{AIDS} & \textbf{Linux} & \textbf{IMDB} & \textbf{ogbg-molhiv} & \textbf{ogbg-code2} & \textbf{ogbg-molpcba} & \textbf{Avg. Rank} \\ 
\midrule
\methodname-\textsc{Mix}  & \cellcolor{shade1}0.80 & \cellcolor{shade1}$\approx1$ & \cellcolor{shade1}$\approx1$ & 0.20 & \cellcolor{shade1}0.12 & \cellcolor{shade3}0.12 &  1.83 \\
    \greed        & \cellcolor{shade2}0.58 & 0.79 & \cellcolor{shade2}0.17 & \cellcolor{shade2}0.23 & \cellcolor{shade2}0.09 & \cellcolor{shade1}0.21 & 2.17 \\ 
         ERIC        & \cellcolor{shade3}0.37 & \cellcolor{shade2}0.92 & \cellcolor{shade3}0.08 & \cellcolor{shade3}0.21 & 0.01 & \cellcolor{shade2}0.18 & 2.83\\ 
     \gedgnn       & 0.35 & \cellcolor{shade3}0.85 & 0.07 & \cellcolor{shade1}0.57 & \cellcolor{shade3} 0.01* & 0.09 & 3.17 \\ 
\bottomrule
\end{tabular}}
\vspace{-0.2in}
\end{table*}

\vspace{-0.1in}
\subsection{Experiment Setup}
\label{sec:setup}
\vspace{-0.05in}
Gemini-1.5 Pro has been used for all experiments. Further details of the software and hardware environments and hyper-parameters used for \methodname are listed in App.~\ref{app:setup}.\\
\textbf{Datasets: }Table~\ref{tab:datasets} summarizes the datasets used in this study. A detailed description of the data semantics is included in App.~\ref{app:datasets}. While AIDS, Linux and IMDB are obtained from \citet{tudataset}, the other four datasets are made available by \citet{ogbg}.\\
\textbf{Benchmark Algorithms:} The recent baselines are listed in Table~\ref{tab:gaps}. From this set, we benchmark \methodname against \greed~\cite{ranjan2022greed}, \gedgnn~\cite{piao2023computing}, \eric~\cite{eric}, \graphedx~\cite{graphedx} and \hmn~\cite{h2mn}. We omit \simgnn, \graphotsim, \gmn, \graphsim, \tagsim and \genn, since they have been outperformed by the considered baselines of \greed, \graphedx, \gedgnn and \eric. 

Among non-neural baselines we include the best-performing heuristics from the benchmarking study in \citet{blumenthal2020comparing}: namely, \textsc{Lp-Ged-F2}, \textsc{Compact-Mip}, \textsc{Adj-IP}, \textsc{Branch-Tight}, \textsc{Node} and \textsc{Ipfp}.

\textbf{\methodname-\textsc{Mix}} is a variant of \methodname trained on a mixture of graph pairs from multiple datasets, while maintaining the same training set size as  \methodname. The programs discovered by \gmix are used for inference across all datasets to assess whether a single training instance can generalize across domains,  eliminating dataset-specific training.


 \begin{table*}[htbp]
\centering
\vspace{-0.1in}
\caption{\textbf{Intra-Domain Generalizability:} RMSE of the best neural method, \greed, and \methodname. Off-diagonal entries represent cross-dataset performance. NA symbolizes that it's not possible to train a single model covering the train-test combination.}
\label{tab:intra-domain}
\setlength{\tabcolsep}{12pt} 
\scalebox{0.85}{
\begin{tabular}{c|cc|cc|cc}
\toprule
\multicolumn{1}{c|}{} & \multicolumn{2}{c|}{\textbf{AIDS}} & \multicolumn{2}{c|}{\textbf{ogbg-molhiv}} & \multicolumn{2}{c}{\textbf{ogbg-molpcba}} \\
\textbf{Test set \textbackslash Train Set} & \greed & \methodname & \greed & \methodname & \greed & \methodname \\
\midrule
\textbf{AIDS} & 0.61 & \cellcolor{shade1}{0.57} & 5.71 & \cellcolor{shade1}0.64 & 4.58 & \cellcolor{shade1}0.59 \\
\textbf{ogbg-molhiv} & NA & \cellcolor{shade1}3.02 & 3.02 & \cellcolor{shade1}{2.96} & 3.86 & \cellcolor{shade1}2.89 \\
\textbf{ogbg-molpcba} & NA & \cellcolor{shade1}3.59 & \cellcolor{shade1}2.16 & 3.54 & \cellcolor{shade1}{2.48} & 3.18 \\
\bottomrule
\end{tabular}}
\vspace{-0.2in}
\end{table*}
\label{par:train_data}
\textbf{Train-Validation-Test Split:} To construct the test set for a particular dataset, we select 1000 graph pairs uniformly at random and compute their true \ged. The procedure for computing the ground truth \ged is discussed in App.~\ref{app:groundtruth}. The training and validation sets depend on the algorithm.
\vspace{-0.15in}
\begin{itemize}
    \item {\bf Neural Algorithms:} All neural approaches are trained on 10,000 graph pairs per dataset. This training time exceeds 15 days for certain datasets (see Fig.~\ref{fig:train_inference_time_comparison}a).
    \vspace{-0.1in}
    \item {\bf \methodname and \gmix:} \methodname is trained with only 1,000 graph pairs per dataset. As discussed already, \methodname does not require ground-truth \ged. In \gmix, we choose 
 $166$ graph pairs from each of the datasets listed in Table~\ref{tab:datasets} except ogbg-ppa. 
 Both \methodname and \gmix do not use a validation set.
    \vspace{-0.1in}
    \item \textbf{Non-Neural Baselines:} These unsupervised algorithms do not require any training or validation datasets.
\end{itemize}
\vspace{-0.1in}


\begin{table*}[htbp]
\centering
\caption{\textbf{Inter-Domain Generalizability: } Generalization of \methodname across domains, dataset sizes, and node label distributions by training on one dataset and measuring RMSE on others. Off-diagonal entries represent cross-dataset performance. The best two results for each test set(row) have been highlighted in shades of green, with darker being better.}
\label{tab:transfer}
\setlength{\tabcolsep}{12pt}
\scalebox{0.85}{
\begin{tabular}{c|cccccc}
\toprule
\textbf{Test set \textbackslash Train set} & \textbf{AIDS} & \textbf{IMDB} & \textbf{Linux} & \textbf{ogbg-molhiv} & \textbf{ogbg-code2} & \textbf{ogbg-molpcba} \\
\midrule
\textbf{AIDS}  & \cellcolor{shade1}0.57 & 0.63 & 0.65 & 0.64 & 0.62 & \cellcolor{shade3}0.59 \\
\textbf{IMDB} & 0.88 & \cellcolor{shade1}0.55 & 0.88 & 0.78 & \cellcolor{shade3}0.74 & 0.87 \\
\textbf{Linux} &0.18 & 0.22 & \cellcolor{shade1}0.13 & 0.24 & \cellcolor{shade3}0.16 & 0.24 \\
\textbf{ogbg-molhiv} & 3.02 & \cellcolor{shade3}2.93 & 3.08 & 2.96 & 2.96 & \cellcolor{shade1}2.89 \\
\textbf{ogbg-code2} & 4.44 & 4.32 & 4.74 & \cellcolor{shade1}4.07 & \cellcolor{shade3}4.22 & 4.5 \\
\textbf{ogbg-molpcba} & 3.59 & 3.63 & 3.61 & \cellcolor{shade3}3.54 &  3.64 & \cellcolor{shade1}3.18\\
\bottomrule
\end{tabular}}
\vspace{-0.2in}
\end{table*}

\textbf{Metrics:} We employ two metrics: Root Mean Squared Error (\textit{RMSE}) and Exact Match Ratio (\textit{EMR}). 
 \textit{EMR} quantifies the proportion of test graph pairs for which the predicted \ged exactly matches the true \ged. (See  App.~\ref{app:metrics} for details.)
\begin{figure*}[btp!]
   \centering
   \subfloat[ogbg-ppa\label{fig:ogbg-ppa}]{
    \includegraphics[width=2.2in]{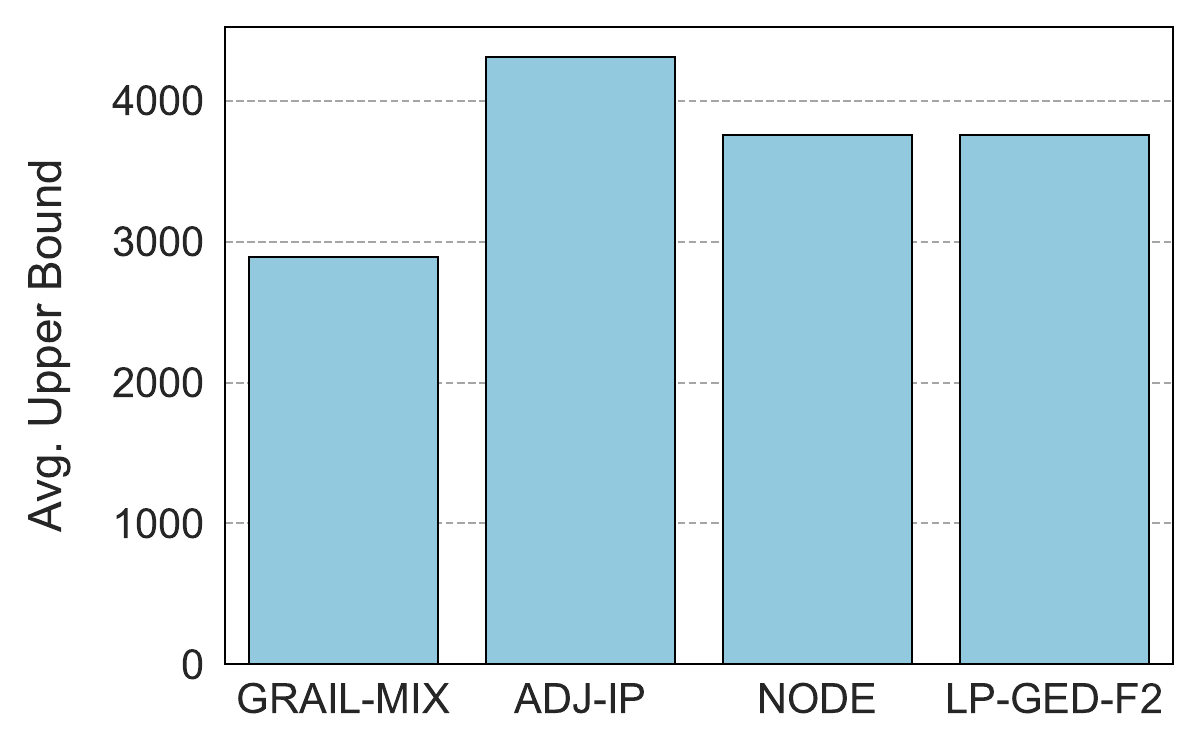}}\hfill
    \subfloat[AIDS \label{fig:avg_ub_aids}]{
    \includegraphics[width=1.65in]{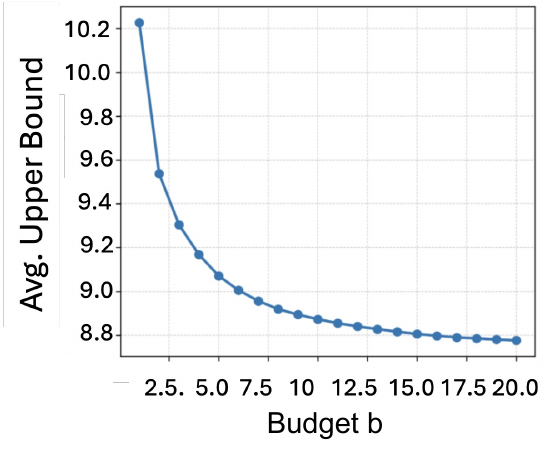}}\hfill
    \subfloat[AIDS \label{fig:submodular}]{
    \includegraphics[width=1.5in]{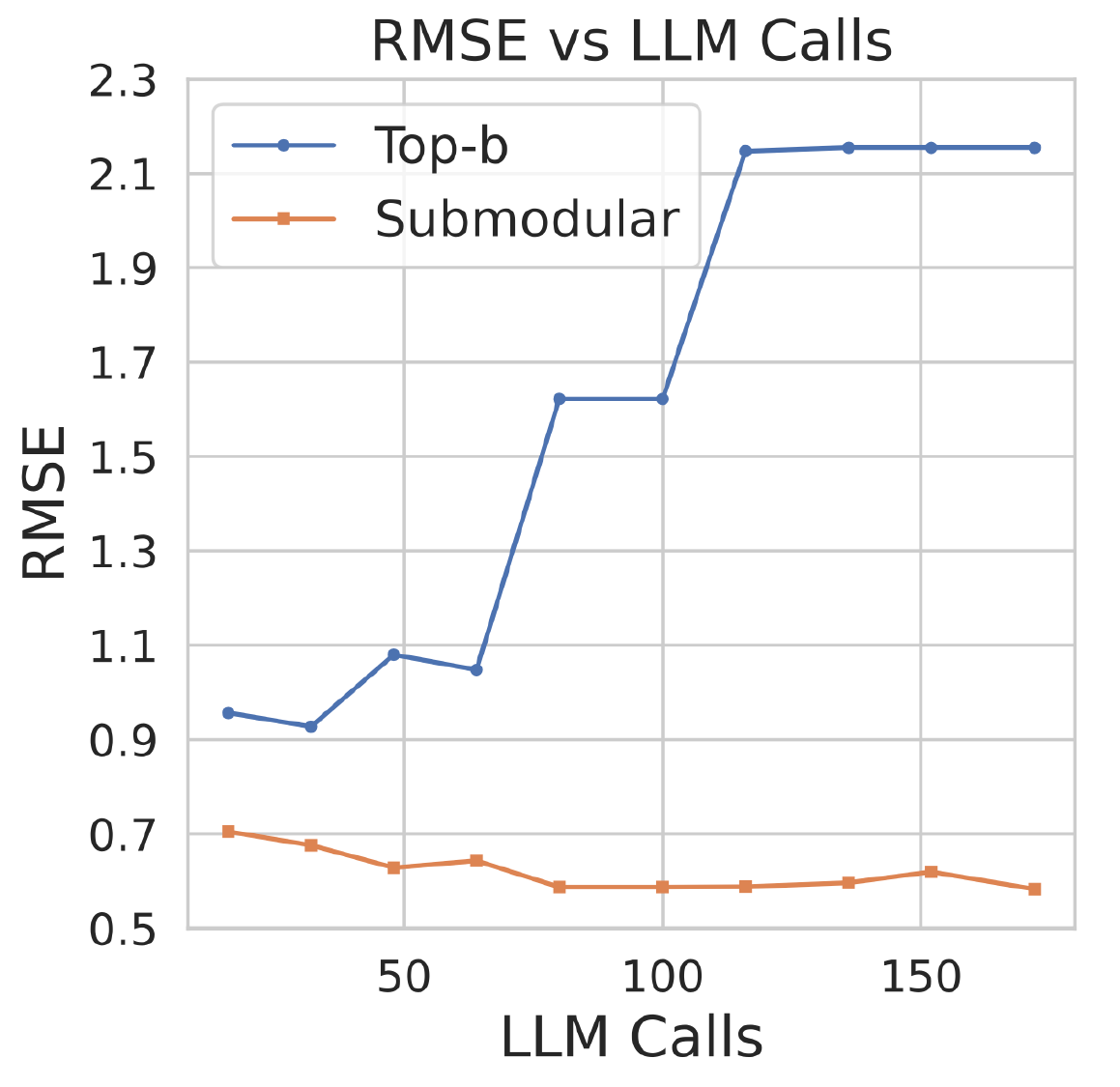}}
    \vspace{-0.1in}
    \caption{(a) \textbf{\gmix at scale: }Performance of \methodname-\textsc{Mix} on the ogbg-ppa dataset when compared to the top-3 non-neural baselines on the basis of average rank in Table ~\ref{tab:rmse}. (b) Impact of function budget on upper bound. (c) Impact of greedy submodular optimization on performance on test set.}
    \vspace{-0.2in}
\end{figure*} 
\vspace{-0.1in}
\subsection{Empirical Analysis of Approximation Errors}
\vspace{-0.05in}
 Tables~\ref{tab:rmse} and \ref{tab:EMR_main} benchmark \methodname in terms of RMSE and EMR. Several important observations emerge.

First, \methodname and \methodname-\textsc{Mix} comprehensively outperform the baselines despite not using any ground truth for training. This is a critical advantage as it saves time in expensive ground truth computation (Refer to Fig.~\ref{fig:train_inference_time_comparison}a). The improvement is the highest in the IMDB dataset, which we specifically analyze in \S~\ref{sec:interpretability}

Second, the efficacy of \methodname-\textsc{Mix} across datasets demonstrates that the discovered programs are universally applicable on multiple datasets and can be called as \textit{foundation functions}. These foundation functions eliminate the need for dataset-specific training.

Third, \gmix outperforms \methodname in three datasets, indicating positive cross-dataset knowledge transfer.
\vspace{-0.1in}
\subsection{Generalizability} 
\vspace{-0.1in}
An intrinsic requirement of all machine learning methods is that the training and test data are sampled from the same distribution. Thus, the neural baselines depend on training data tailored to the test dataset, limiting their ability to transfer knowledge due to reliance on dataset-specific features. 
In contrast, \methodname learns symbolic logical rules in the form of programs, facilitating out-of-domain and out-of-distribution generalization. We now evaluate this capability.\\
\textbf{Intra-domain: }Neural models are limited by feature dimensionality, making zero-shot generalization infeasible. However, for domains such as chemical compounds, a uniform feature space allows training a single model. We retrain the best neural baseline, \greed, and compare its intra-domain generalizability with \methodname in Table~\ref{tab:intra-domain}. Note that the AIDS dataset has a smaller feature dimension than the ogbg datasets, making it impossible to derive a common feature space. \methodname generalizes well across all train-test dataset pairs, while \greed struggles except for ogbg-molhiv and ogbg-molpcba. This is because both datasets are adopted from the same parent dataset MoleculeNet~\cite{wu2018moleculenet, hu2020open} with similar topological features (Table~\ref{tab:datasets}). \\
\textbf{Inter-domain: }
Table~\ref{tab:transfer} showcases the ability of the functions discovered by dataset-specific training of \methodname to generalize across other datasets. We do not observe a significant increase in RMSE on the off-diagonal entries, which showcases positive knowledge transfer and an ability not seen in neural approximators.\\
\textbf{Generalization to graph size:}  
In this experiment, we evaluate \gmix on the ogbg-ppa dataset, which has been omitted from prior benchmarking due to the large size of its graphs (See Table~\ref{tab:datasets}). Given the computational infeasibility of ground-truth \ged computation for these graphs, neural approximators cannot be trained on this dataset. To assess the generalization capability of \gmix, we compare the upper bound provided by its programs against the top-3 non-neural heuristics (based on Table~\ref{tab:rmse}). We observe that \gmix provides $30\%$ to $45\%$ tighter upper bounds. Fig.~\ref{fig:rmse_vs_graph_size} in the appendix further substantiates that \name generalizes better to large graphs than neural baselines.


\vspace{-0.1in}
\subsection{Ablation Study and Parameters}
\label{sec:ablation}
\vspace{-0.05in}
\textbf{Impact of budget $b$:} For the true \ged, all possible mappings (factorial in the graph size) must be considered. Instead, \methodname restricts this to $b$ mappings, where each mapping is generated by a program. 
In Fig.~\ref{fig:avg_ub_aids}, we plot how $b$ affects the upper bound. As shown, the upper bound converges at $\approx 15$ functions. Similar trends are observed in other datasets (see Fig.~\ref{fig:avg_upper_bound}).

\textbf{Impact of Submodularity:} What happens if, instead of selecting the top-$b$ functions using greedy submodular optimization, we evolve and score functions individually based on their upper bounds (Eq.~\ref{eq:gedp}) and select the top-$b$ solely based on this criterion? Fig.~\ref{fig:submodular} illustrates the impact on RMSE across training iterations (LLM calls). While selecting the top-$b$ functions through submodular optimization shows a clear trend of decreasing RMSE on the test set, independently choosing the top-$b$ functions based on individual scores results in significantly higher RMSE, with a progressive worsening trend indicative of overfitting to the training set. This result underscores the importance of submodularity in selecting functions that complement one another and perform well collectively (See Fig.~\ref{fig:topk_vs_greedy_aids} for additional metrics).

\vspace{-0.15in}
\subsection{Interpretability: Case Study on IMDB} 
\label{sec:casestudy}
\vspace{-0.05in}
To shed light on the superior performance of \name over GNN-based neural approximators, we analyze a graph pair from the IMDB dataset, where \name shows the highest improvement over all baselines (Table \ref{tab:rmse}). Figures \ref{fig:casestudy_graph} and \ref{fig:casestudy_heatmap} in the Appendix illustrate the graph structures, the edits made by \name and its closest competitor \gedgnn, and the node similarity matrices generated by these algorithms. The program discovered by \gmix, shown in Fig.~\ref{fig:casestudy_program}, achieves the ground truth GED of 4, while \gedgnn predicts a GED of 20. This program assigns node similarity scores based on degree similarity and that of their neighbors. Since IMDB is unlabeled, feature similarity does not influence the results.

Examining \gedgnn's similarity matrix reveals a different score distribution compared to \name. For instance, node 1 in Graph 1 has the second-highest similarity to nodes 2 and 9 in \gmix, but \gedgnn assigns low similarity to these nodes, favoring nodes 0 and 6 instead. \gmix's decision aligns with the similarity in their degrees (degree of 9 for node 1 in Graph 1 versus 7 for nodes 2 and 9 in Graph 2). In contrast, \gedgnn, as a neural network, operates as a black box. We hypothesize that the poor performance of \gedgnn and other GNN-based algorithms in IMDB is due to the dataset's unlabeled nature and high density, leading to oversquashing.~\cite{oversquashing}.
\label{sec:interpretability}

\vspace{-0.1in}
\section{ Conclusions and Future Directions}
\label{sec: conclusion}
\vspace{-0.05in}
This paper introduced a new paradigm of computing \ged by leveraging LLMs to autonomously generate programs. Unlike traditional methods that rely on neural networks and require computationally expensive, NP-hard ground truth data, our method employs a self-evolutionary strategy to discover programs without any ground truth data. Remarkably, these programs not only surpass state-of-the-art methods on average but are also interpretable and demonstrate strong transferability across various datasets and domains. While our approach is demonstrated on GED computation, we believe it is generalizable to other combinatorial problems with similar constraints, both within and beyond graph-related tasks. An interesting direction for future work is to critically analyze the programs discovered by our method with domain experts and to develop mechanisms that facilitate closer cooperation between human and LLM agents.

\clearpage
\section*{Acknowledgements}
\vspace{-0.1in}
The authors extend our sincere gratitude to Alexander Novikov (Research scientist, Google DeepMind) for their thorough review of the initial draft of this work. Their insightful comments and constructive feedback were instrumental in refining the manuscript and enhancing its clarity.

\section*{Statement of Impact}
This work introduces a novel approach to Graph Edit Distance (\ged) computation by leveraging large language models (LLMs) and evolutionary algorithms to generate interpretable programs for \ged approximation. Unlike existing methods, our approach prioritizes transparency, interpretability, and cross-domain generalization while achieving state-of-the-art performance in approximation accuracy.

The societal implications of this work are significant. By addressing key limitations of neural approaches—such as their reliance on costly ground truth data, lack of interpretability, and domain-specific retraining—our method has the potential to make graph similarity computation more accessible and efficient across a variety of applications, including bioinformatics, social network analysis, and cheminformatics. Moreover, the transparency of program-based solutions could foster trust and reliability in critical domains where understanding the computation process is essential, such as healthcare or legal systems.

We do not foresee any ethical concerns arising out of our work.

\bibliographystyle{icml2025}
\bibliography{icml}

\begin{thebibliography}{41}
\providecommand{\natexlab}[1]{#1}
\providecommand{\url}[1]{\texttt{#1}}
\expandafter\ifx\csname urlstyle\endcsname\relax
  \providecommand{\doi}[1]{doi: #1}\else
  \providecommand{\doi}{doi: \begingroup \urlstyle{rm}\Url}\fi

\bibitem[Bai \& Zhao(2021)Bai and Zhao]{tagsim}
Bai, J. and Zhao, P.
\newblock Tagsim: type-aware graph similarity learning and computation.
\newblock \emph{Proc. VLDB Endow.}, 15\penalty0 (2):\penalty0 335–347, October 2021.
\newblock ISSN 2150-8097.
\newblock \doi{10.14778/3489496.3489513}.
\newblock URL \url{https://doi.org/10.14778/3489496.3489513}.

\bibitem[Bai et~al.(2019)Bai, Ding, Bian, Chen, Sun, and Wang]{simgnn}
Bai, Y., Ding, H., Bian, S., Chen, T., Sun, Y., and Wang, W.
\newblock Simgnn: A neural network approach to fast graph similarity computation.
\newblock In \emph{WSDM}, WSDM '19, pp.\  384–392, 2019.

\bibitem[Bai et~al.(2020)Bai, Ding, Gu, Sun, and Wang]{graphsim}
Bai, Y., Ding, H., Gu, K., Sun, Y., and Wang, W.
\newblock Learning-based efficient graph similarity computation via multi-scale convolutional set matching.
\newblock \emph{AAAI}, pp.\  3219--3226, Apr. 2020.

\bibitem[Blumenthal(2019)]{gedthesis}
Blumenthal, D.~B.
\newblock New techniques for graph edit distance computation.
\newblock \emph{CoRR}, abs/1908.00265, 2019.
\newblock URL \url{http://arxiv.org/abs/1908.00265}.

\bibitem[Blumenthal \& Gamper(2018)Blumenthal and Gamper]{BRANCH_TIGHT}
Blumenthal, D.~B. and Gamper, J.
\newblock Improved lower bounds for graph edit distance.
\newblock \emph{IEEE Transactions on Knowledge and Data Engineering}, 30\penalty0 (3):\penalty0 503--516, 2018.
\newblock \doi{10.1109/TKDE.2017.2772243}.

\bibitem[Blumenthal \& Gamper(2020)Blumenthal and Gamper]{COMPACT_MIP}
Blumenthal, D.~B. and Gamper, J.
\newblock On the exact computation of the graph edit distance.
\newblock \emph{Pattern Recogn. Lett.}, 134\penalty0 (C):\penalty0 46–57, June 2020.
\newblock ISSN 0167-8655.
\newblock \doi{10.1016/j.patrec.2018.05.002}.
\newblock URL \url{https://doi.org/10.1016/j.patrec.2018.05.002}.

\bibitem[Blumenthal et~al.(2020)Blumenthal, Boria, Gamper, Bougleux, and Brun]{blumenthal2020comparing}
Blumenthal, D.~B., Boria, N., Gamper, J., Bougleux, S., and Brun, L.
\newblock Comparing heuristics for graph edit distance computation.
\newblock \emph{The VLDB journal}, 29\penalty0 (1):\penalty0 419--458, 2020.

\bibitem[Bommakanti et~al.(2024{\natexlab{a}})Bommakanti, Vonteri, Ranu, and Karras]{eugene}
Bommakanti, A., Vonteri, H.~R., Ranu, S., and Karras, P.
\newblock Eugene: Explainable unsupervised approximation of graph edit distance, 2024{\natexlab{a}}.
\newblock URL \url{https://arxiv.org/abs/2402.05885}.

\bibitem[Bommakanti et~al.(2024{\natexlab{b}})Bommakanti, Vonteri, Skitsas, Ranu, Mottin, and Karras]{fugal}
Bommakanti, A., Vonteri, H.~R., Skitsas, K., Ranu, S., Mottin, D., and Karras, P.
\newblock Fugal: Feature-fortified unrestricted graph alignment.
\newblock In \emph{The Thirty-eighth Annual Conference on Neural Information Processing Systems}, 2024{\natexlab{b}}.

\bibitem[Chen et~al.(2019)Chen, Peng, Han, Cai, and Cai]{chen2019hogmmnc}
Chen, J., Peng, H., Han, G., Cai, H., and Cai, J.
\newblock Hogmmnc: a higher order graph matching with multiple network constraints model for gene--drug regulatory modules identification.
\newblock \emph{Bioinformatics}, 35\penalty0 (4):\penalty0 602--610, 2019.

\bibitem[Conte et~al.(2003)Conte, Foggia, Sansone, and Vento]{Conte2003GraphMA}
Conte, D., Foggia, P., Sansone, C., and Vento, M.
\newblock Graph matching applications in pattern recognition and image processing.
\newblock \emph{Proceedings 2003 International Conference on Image Processing (Cat. No.03CH37429)}, 2:\penalty0 II--21, 2003.
\newblock URL \url{https://api.semanticscholar.org/CorpusID:267842199}.

\bibitem[Cormen et~al.(2009)Cormen, Leiserson, Rivest, and Stein]{cormen}
Cormen, T.~H., Leiserson, C.~E., Rivest, R.~L., and Stein, C.
\newblock \emph{Introduction to Algorithms, Third Edition}.
\newblock The MIT Press, 3rd edition, 2009.
\newblock ISBN 0262033844.

\bibitem[Doan et~al.(2021)Doan, Manchanda, Mahapatra, and Reddy]{graphotsim}
Doan, K.~D., Manchanda, S., Mahapatra, S., and Reddy, C.~K.
\newblock Interpretable graph similarity computation via differentiable optimal alignment of node embeddings.
\newblock In \emph{SIGIR}, pp.\  665–674, 2021.

\bibitem[Fan et~al.(2020)Fan, Mao, Wu, and Xu]{grampa}
Fan, Z., Mao, C., Wu, Y., and Xu, J.
\newblock Spectral graph matching and regularized quadratic relaxations: Algorithm and theory.
\newblock In III, H.~D. and Singh, A. (eds.), \emph{Proceedings of the 37th International Conference on Machine Learning}, volume 119 of \emph{Proceedings of Machine Learning Research}, pp.\  2985--2995. PMLR, 13--18 Jul 2020.

\bibitem[Giovanni et~al.(2024)Giovanni, Rusch, Bronstein, Deac, Lackenby, Mishra, and Veli{\v{c}}kovi{\'c}]{oversquashing}
Giovanni, F.~D., Rusch, T.~K., Bronstein, M., Deac, A., Lackenby, M., Mishra, S., and Veli{\v{c}}kovi{\'c}, P.
\newblock How does over-squashing affect the power of {GNN}s?
\newblock \emph{Transactions on Machine Learning Research}, 2024.
\newblock ISSN 2835-8856.
\newblock URL \url{https://openreview.net/forum?id=KJRoQvRWNs}.

\bibitem[Gordon \& Whitley(1993)Gordon and Whitley]{island}
Gordon, V.~S. and Whitley, L.~D.
\newblock Serial and parallel genetic algorithms as function optimizers.
\newblock In \emph{Proceedings of the 5th International Conference on Genetic Algorithms}, pp.\  177–183, San Francisco, CA, USA, 1993. Morgan Kaufmann Publishers Inc.
\newblock ISBN 1558602992.

\bibitem[He \& Singh(2006)He and Singh]{ctree}
He, H. and Singh, A.
\newblock Closure-tree: An index structure for graph queries.
\newblock In \emph{22nd International Conference on Data Engineering (ICDE'06)}, pp.\  38--38, 2006.
\newblock \doi{10.1109/ICDE.2006.37}.

\bibitem[Hopcroft \& Karp(1973)Hopcroft and Karp]{hopcroft1973n}
Hopcroft, J.~E. and Karp, R.~M.
\newblock An n\^{}5/2 algorithm for maximum matchings in bipartite graphs.
\newblock \emph{SIAM Journal on Computing}, 2\penalty0 (4):\penalty0 225--231, 1973.

\bibitem[Hu et~al.(2020)Hu, Fey, Zitnik, Dong, Ren, Liu, Catasta, and Leskovec]{hu2020open}
Hu, W., Fey, M., Zitnik, M., Dong, Y., Ren, H., Liu, B., Catasta, M., and Leskovec, J.
\newblock Open graph benchmark: Datasets for machine learning on graphs.
\newblock In Larochelle, H., Ranzato, M., Hadsell, R., Balcan, M.~F., and Lin, H. (eds.), \emph{Advances in Neural Information Processing Systems}, volume~33, pp.\  22118--22133. Curran Associates, Inc., 2020.

\bibitem[Hu et~al.(2021)Hu, Fey, Zitnik, Dong, Ren, Liu, Catasta, and Leskovec]{ogbg}
Hu, W., Fey, M., Zitnik, M., Dong, Y., Ren, H., Liu, B., Catasta, M., and Leskovec, J.
\newblock Open graph benchmark: Datasets for machine learning on graphs, 2021.
\newblock URL \url{https://arxiv.org/abs/2005.00687}.

\bibitem[Jain et~al.(2024)Jain, Roy, Meher, Chakrabarti, and De]{graphedx}
Jain, E., Roy, I., Meher, S., Chakrabarti, S., and De, A.
\newblock Graph edit distance with general costs using neural set divergence.
\newblock In \emph{The Thirty-eighth Annual Conference on Neural Information Processing Systems}, 2024.

\bibitem[Justice \& Hero(2006)Justice and Hero]{NODE_ADJ_IP}
Justice, D. and Hero, A.
\newblock A binary linear programming formulation of the graph edit distance.
\newblock \emph{IEEE Transactions on Pattern Analysis and Machine Intelligence}, 28\penalty0 (8):\penalty0 1200--1214, 2006.
\newblock \doi{10.1109/TPAMI.2006.152}.

\bibitem[Kuhn(1955)]{kuhn1955hungarian}
Kuhn, H.~W.
\newblock The hungarian method for the assignment problem.
\newblock \emph{Naval Research Logistics Quarterly}, 2\penalty0 (1-2):\penalty0 83--97, 1955.

\bibitem[Leordeanu et~al.(2009)Leordeanu, Hebert, and Sukthankar]{leordeanu2009integer}
Leordeanu, M., Hebert, M., and Sukthankar, R.
\newblock An integer projected fixed point method for graph matching and map inference.
\newblock \emph{Advances in neural information processing systems}, 22, 2009.

\bibitem[Lerouge et~al.(2017{\natexlab{a}})Lerouge, Abu-Aisheh, Raveaux, H{\'e}roux, and Adam]{lerouge2017new}
Lerouge, J., Abu-Aisheh, Z., Raveaux, R., H{\'e}roux, P., and Adam, S.
\newblock New binary linear programming formulation to compute the graph edit distance.
\newblock \emph{Pattern Recognition}, 72:\penalty0 254--265, 2017{\natexlab{a}}.

\bibitem[Lerouge et~al.(2017{\natexlab{b}})Lerouge, Abu-Aisheh, Raveaux, Héroux, and Adam]{MIP-F2}
Lerouge, J., Abu-Aisheh, Z., Raveaux, R., Héroux, P., and Adam, S.
\newblock New binary linear programming formulation to compute the graph edit distance.
\newblock \emph{Pattern Recognition}, 72:\penalty0 254--265, 2017{\natexlab{b}}.
\newblock ISSN 0031-3203.
\newblock \doi{https://doi.org/10.1016/j.patcog.2017.07.029}.
\newblock URL \url{https://www.sciencedirect.com/science/article/pii/S003132031730300X}.

\bibitem[Li et~al.(2019)Li, Gu, Dullien, Vinyals, and Kohli]{icmlged}
Li, Y., Gu, C., Dullien, T., Vinyals, O., and Kohli, P.
\newblock Graph matching networks for learning the similarity of graph structured objects.
\newblock In \emph{ICML}, pp.\  3835--3845, 2019.

\bibitem[Morris et~al.(2020)Morris, Kriege, Bause, Kersting, Mutzel, and Neumann]{tudataset}
Morris, C., Kriege, N.~M., Bause, F., Kersting, K., Mutzel, P., and Neumann, M.
\newblock Tudataset: {A} collection of benchmark datasets for learning with graphs.
\newblock \emph{CoRR}, abs/2007.08663, 2020.
\newblock URL \url{https://arxiv.org/abs/2007.08663}.

\bibitem[Piao et~al.(2023)Piao, Xu, Sun, Rong, Zhao, and Cheng]{piao2023computing}
Piao, C., Xu, T., Sun, X., Rong, Y., Zhao, K., and Cheng, H.
\newblock Computing graph edit distance via neural graph matching.
\newblock \emph{Proceedings of the VLDB Endowment}, 16\penalty0 (8):\penalty0 1817--1829, 2023.

\bibitem[Ranjan et~al.(2022)Ranjan, Grover, Medya, Chakaravarthy, Sabharwal, and Ranu]{ranjan2022greed}
Ranjan, R., Grover, S., Medya, S., Chakaravarthy, V., Sabharwal, Y., and Ranu, S.
\newblock Greed: A neural framework for learning graph distance functions.
\newblock \emph{Advances in Neural Information Processing Systems}, 35:\penalty0 22518--22530, 2022.

\bibitem[Ranu et~al.(2014)Ranu, Hoang, and Singh]{divquery}
Ranu, S., Hoang, M., and Singh, A.
\newblock Answering top-k representative queries on graph databases.
\newblock In \emph{Proceedings of the 2014 ACM SIGMOD international conference on Management of data}, pp.\  1163--1174, 2014.

\bibitem[Romera-Paredes et~al.(2024)Romera-Paredes, Barekatain, Novikov, Balog, Kumar, Dupont, Ruiz, Ellenberg, Wang, Fawzi, et~al.]{romera2024mathematical}
Romera-Paredes, B., Barekatain, M., Novikov, A., Balog, M., Kumar, M.~P., Dupont, E., Ruiz, F.~J., Ellenberg, J.~S., Wang, P., Fawzi, O., et~al.
\newblock Mathematical discoveries from program search with large language models.
\newblock \emph{Nature}, 625\penalty0 (7995):\penalty0 468--475, 2024.

\bibitem[Singh et~al.(2008)Singh, Xu, and Berger]{doi:10.1073/pnas.0806627105}
Singh, R., Xu, J., and Berger, B.
\newblock Global alignment of multiple protein interaction networks with application to functional orthology detection.
\newblock \emph{Proceedings of the National Academy of Sciences}, 105\penalty0 (35):\penalty0 12763--12768, 2008.
\newblock \doi{10.1073/pnas.0806627105}.
\newblock URL \url{https://www.pnas.org/doi/abs/10.1073/pnas.0806627105}.

\bibitem[Szklarczyk et~al.(2019)Szklarczyk, Gable, Lyon, Junge, Wyder, Huerta-Cepas, Simonovic, Doncheva, Morris, and et~al.]{szklarczyk2019string}
Szklarczyk, D., Gable, A.~L., Lyon, D., Junge, A., Wyder, S., Huerta-Cepas, J., Simonovic, M., Doncheva, N.~T., Morris, J.~H., and et~al., P.~B.
\newblock String v11: protein–protein association networks with increased coverage, supporting functional discovery in genome-wide experimental datasets.
\newblock \emph{Nucleic Acids Research}, 47\penalty0 (D1):\penalty0 D607--D613, 2019.

\bibitem[Wang et~al.(2021)Wang, Zhang, Yu, Yan, and Yang]{genn}
Wang, R., Zhang, T., Yu, T., Yan, J., and Yang, X.
\newblock Combinatorial learning of graph edit distance via dynamic embedding.
\newblock In \emph{IEEE Conference on Computer Vision and Pattern Recognition}, 2021.

\bibitem[Wang et~al.(2012)Wang, Ding, Tung, Ying, and Jin]{wang2012efficient}
Wang, X., Ding, X., Tung, A. K.~H., Ying, S., and Jin, H.
\newblock An efficient graph indexing method.
\newblock In \emph{Proceedings of the 2012 IEEE 28th International Conference on Data Engineering (ICDE '12)}, pp.\  210--221, USA, 2012. IEEE Computer Society.

\bibitem[Wu et~al.(2018)Wu, Ramsundar, Feinberg, Gomes, Geniesse, Pappu, Leswing, and Pande]{wu2018moleculenet}
Wu, Z., Ramsundar, B., Feinberg, E.~N., Gomes, J., Geniesse, C., Pappu, A.~S., Leswing, K., and Pande, V.
\newblock Moleculenet: a benchmark for molecular machine learning.
\newblock \emph{Chemical science}, 9\penalty0 (2):\penalty0 513--530, 2018.

\bibitem[Yanardag \& Vishwanathan(2015)Yanardag and Vishwanathan]{yanardag2015deep}
Yanardag, P. and Vishwanathan, S.
\newblock Deep graph kernels.
\newblock In \emph{Proceedings of the 21st ACM SIGKDD International Conference on Knowledge Discovery and Data Mining (KDD '15)}, pp.\  1365--1374, New York, NY, USA, 2015. Association for Computing Machinery.

\bibitem[Zhang et~al.(2021)Zhang, Bu, Ester, Li, Yao, Yu, and Wang]{h2mn}
Zhang, Z., Bu, J., Ester, M., Li, Z., Yao, C., Yu, Z., and Wang, C.
\newblock H2mn: Graph similarity learning with hierarchical hypergraph matching networks.
\newblock In \emph{Proceedings of the 27th ACM SIGKDD Conference on Knowledge Discovery \& Data Mining}, KDD '21, pp.\  2274–2284, New York, NY, USA, 2021. Association for Computing Machinery.
\newblock ISBN 9781450383325.
\newblock \doi{10.1145/3447548.3467328}.
\newblock URL \url{https://doi.org/10.1145/3447548.3467328}.

\bibitem[Zhuo \& Tan(2022)Zhuo and Tan]{eric}
Zhuo, W. and Tan, G.
\newblock Efficient graph similarity computation with alignment regularization.
\newblock \emph{Advances in Neural Information Processing Systems}, 35:\penalty0 30181--30193, 2022.

\bibitem[Zitnik et~al.(2019)Zitnik, Feldman, and et~al.]{zitnik2019evolution}
Zitnik, M., Feldman, M.~W., and et~al., J.~L.
\newblock Evolution of resilience in protein interactomes across the tree of life.
\newblock \emph{Proceedings of the National Academy of Sciences}, 116\penalty0 (10):\penalty0 4426--4433, 2019.

\end{thebibliography}

\newpage
\appendix
\onecolumn
\section{Appendix}
\subsection{Proofs}
\subsubsection{NP-hardness of Eq.~\ref{eq:minA}}
\label{app:setcover}
\section*{Reduction to Prove NP-Hardness}

We reduce the \textit{Set Cover} problem to the given problem in polynomial time to demonstrate its NP-hardness.

Given a universe of elements \( U = \{e_1, e_2, \ldots, e_n\} \), a collection of sets \( \mathcal{S} = \{S_1, S_2, \ldots, S_m\} \) where \( S_i \subseteq U \), and a budget \( b \), the set cover problem seeks to determine if there exist \( b \) sets \( S_1, \ldots, S_b \in \mathcal{S} \) whose union covers all elements of \( U \).

Given an instance of the set cover problem, we construct a bipartite graph $\CB=(\CV,\CU,\CE,\CW)$, where $\CV=\CS$, $\CU=U$, and an edge \( (S_i, e_j) \in \CE \) exists if and only if \( S_i \) covers \( e_j \). Each edge \( (S_i, e_j) \) has a weight:
\[
w(S_i, e_j) =
\begin{cases} 
1 & \text{if } S_i \text{ covers } e_j, \\
1 + \Delta & \text{if } S_i \text{ does not cover } e_j.
\end{cases}
\]
where \( \Delta > 0 \).

The objective is to select \( b \) nodes from \( \CV \) (representing sets $\CS$) such that Eq.~\ref{eq:minA} is minimized on graph $\CB$.

 If a \textit{Set Cover} of size \( b \) exists, then all \( n \) elements can be covered by \( b \) sets. This means if we select the corresponding nodes $\CA^*\subseteq \CV$, then every node in $\CU$ will have at least one edge of weight $1$ from some node in $\CA^*$ incident on it. Hence, $\CJ(\CA^*)$ will return a cumulative sum of $n$. 
 
Conversely, if no \textit{Set Cover} of size \( b \) exists, then some elements will not be covered by the selected sets, and their corresponding nodes in $\CU$ will have only edges of edge weights \( 1 + \Delta \) from nodes in $\CA^*$.

Therefore, a solution to the \textit{Set Cover} problem exists iff selecting the corresponding nodes $\CA^*\subseteq \CV$ leads to \(\CJ(\CA^*)= n \). Conversely, if \(\CJ(\CA^*)> n \), it implies that no \( b \)-set cover exists.

\subsubsection{Monotonicity and Submodularity}
\label{app:submodularity}
\begin{lemma}
Monotonicity: $\CJ(\CA)\leq \CJ(\CA')$ if $\CA\supseteq\CA'$.
\end{lemma}
\begin{proof} Since $\CJ(\CA')$ computes minimum over all available mappings in $\CA'$, the minimum can only reduce when additional mappings are added to form $\CA$.
\end{proof}
\begin{lemma}
Submodularity: $\CJ(\CA\cup \{P\})-\CJ(\CA) \leq  \CJ(\CA'\cup \{P\})-\CJ(\CA')$.
\end{lemma}
\begin{proof} We seek to show that the marginal reduction in $\CJ(\CA)$ when a program (mapping) $P$ is added to $\CA$ is atmost as large as adding $P$ to its subset $\CA'$. We establish this through \textit{proof by contradiction}.

Let us assume 
\begin{equation}
\label{eq:subcontra}
\exists \CA\supseteq\CA',\:\CJ(\CA\cup\{P\})-\CJ(\CA)> \CJ(\CA'\cup\{P\})-\CJ(\CA')
\end{equation}
Due to the min operator in Eq.~\ref{eq:J}, Eq.~\ref{eq:subcontra} implies that the additional number of graph pairs where $P$ contributes to the minimum mapping is higher when added to $\CA$ than when added to $\CA'$. This creates a contradiction, since if $P$ contributes to the minimum of a graph pair in $\CA\cup\{P\}$, then it guaranteed to contribute to the minimum for the same pair in $\CA'\cup\{P\}$ as well.
\end{proof}
\label{app:prompt}
{
   \begin{figure}[h!]
   \hspace{1.5cm}
    \includegraphics[scale=0.67]{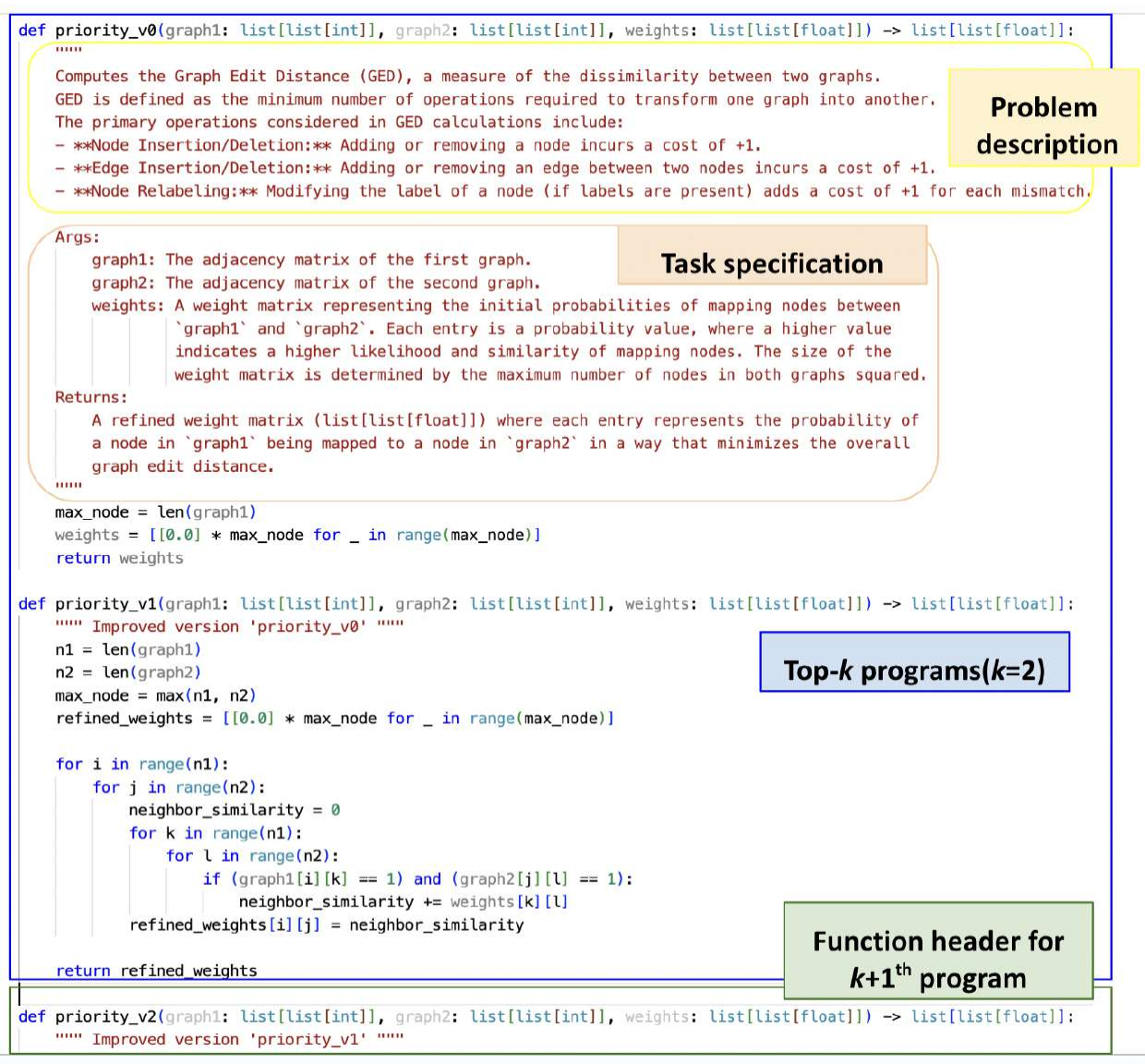}
    \caption{Example of an input prompt to \methodname}
    \label{fig:prompt}
\end{figure} 
\subsection{Experiments}
\subsubsection{Setup}
\label{app:setup}
All experiments ran on a machine equipped with an Intel Xeon Gold 6142 CPU @1GHz and a GeForce GTX 1080 Ti GPU. While non-neural methods and \methodname run on the CPU, neural baselines exploit the GPU. For the LLM,
we use \href{https://cloud.google.com/vertex-ai/generative-ai/docs/learn/models#gemini-1.5-pro}{Gemini 1.5 Pro}. In particular, we have used the initial stable version of  Gemini 1.5 Pro, i.e., gemini-1.5-pro-001, which was released on May 24, 2024. 


\textbf{Hyper-parameters:}
\begin{table}[b!]
\centering
\renewcommand{\thetable}{H} 
\begin{tabular}{l|r}
\toprule
Hyper-parameter & Value \\ \midrule
$k$          & 2         \\ 
$b$ &15 \\
number of islands &5\\
temperature &0.99\\
Algorithm for bipartite matching & Neighbor-biased mapper~\cite{ctree}\\
\bottomrule
\end{tabular}
\vspace{-0.1in}
\caption{Hyper-parameters used for \methodname}
\label{tab:parameters}
\end{table}
\begin{figure}
    \includegraphics[scale=0.47]{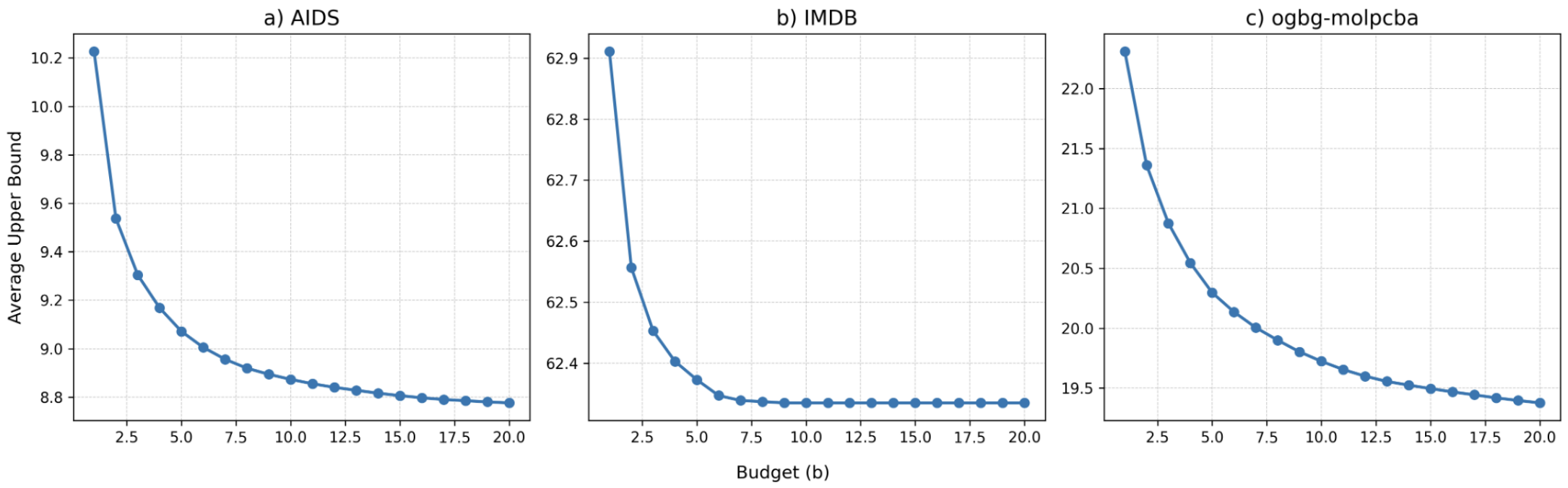}
    \caption{Avg. Upper Bound vs function budget (b) for submodular greedy selection}
    \label{fig:avg_upper_bound}
\end{figure} 
Table~\ref{tab:parameters} lists the hyper-parameters used for \methodname. $k$ stands for the number of functions per response generated by the LLM and $b$ is the function budget employed for submodularity while training. 
We decided to use $k$ as 2, since with greater values of k, we observed no significant improvement in quality metrics. This was also observed in FunSearch~\cite{romera2024mathematical}. For the function budget $b$, we observed that a value of 15 was good enough for most datasets (Refer Fig.~\ref{fig:avg_upper_bound}).

\subsection{Datasets}
\label{app:datasets}
The semantics of these datasets are as follows:
\begin{itemize}
    \item \textbf{ogbg-molhiv and ogbg-molpcba:} These are chemical compound datasets, with each graph representing a molecule. Nodes in these graphs correspond to atoms and are labeled with their atomic numbers, while edges denote the chemical bonds between atoms. These datasets vary in size and complexity, with a rich diversity of molecular structures, enabling us to test the robustness and generalizability of our method.
    \item \textbf{ogbg-code2:} This dataset comprises a vast collection of Abstract Syntax Trees (ASTs) generated from nearly 450,000 Python method definitions. Each graph in this dataset represents an AST, with nodes labeled from a predefined set of 97 categories, capturing various syntactic constructs within the methods. These graphs are considered undirected, simplifying the representation while preserving structural relationships.
    \item \textbf{ogbg-ppa:} This dataset includes undirected protein association neighborhoods extracted from protein-protein interaction networks of 1,581 species\cite{szklarczyk2019string} across 37 diverse taxonomic groups. To build these neighborhoods, 100 proteins were randomly selected from each species, and 2-hop protein association neighborhoods were constructed around each selected protein\cite{zitnik2019evolution}. In these graphs, proteins are represented as nodes, and edges indicate biologically relevant associations between them.
    \item \textbf{AIDS:} This dataset is a collection of graphs sourced from the AIDS antiviral screen database, each graph representing a chemical compound's molecular structure. These graphs are labeled, capturing meaningful properties of the compounds, and are compact in size, containing no more than 10 nodes.
    \item \textbf{Linux:} A collection of program dependence graphs where nodes correspond to statements and edges indicate dependencies between statements. The graph sizes in this dataset are also limited to 10 nodes. This dataset is unlabeled and was introduced in \cite{wang2012efficient}.
    \item \textbf{IMDB:} This dataset consists of ego-networks of actors and actresses who have shared screen time in movies. Each graph represents an ego-network where the nodes correspond to individuals (actors or actresses), and the edges denote shared appearances in films. This dataset is unlabeled and was introduced in \cite{yanardag2015deep}.
\end{itemize}
\subsubsection{Ground-truth data generation}
\label{app:groundtruth}
We employ \textsc{Mip}-F2~\cite{MIP-F2} to generate ground truth GED. \textsc{Mip}-F2 returns the lower and upper bounds of \ged. We compute these bounds with a time limit of 600 seconds per pair. Pairs with equal lower and upper bounds are included in the ground truth.

\subsubsection{Metrics}
\label{app:metrics}
We use the following two metrics to quantify accuracy:
\begin{itemize}
    \item \textbf{RMSE:} Evaluates the prediction accuracy by measuring the
disparities between actual and predicted values. For $n$ graph pairs,
it is defined as: 
\begin{center}
    $\sqrt{\frac{1}{n} \sum_{i=1}^n (\text{true-ged}_i - \text{pred-ged}_i)^2}$
\end{center}
\item \textbf{Exact Match Ratio:} Represents the proportion of graph pairs where the predicted GED exactly matches the actual GED. For \( n \) graph pairs, it is defined as:  
\begin{center}
    $\frac{1}{n} \sum_{i=1}^n \mathbb{I}\big(\text{true-ged}_i = \text{pred-ged}_i\big)$
\end{center}
where \(\mathbb{I}(\cdot)\) is an indicator function that returns 1 if the condition inside is true, and 0 otherwise. A higher Exact Match Ratio indicates better predictive accuracy at the individual graph pair level.
\end{itemize}
\comment{
\subsubsection{Test Data:} The statistics of the test data used for the evaluation of \methodname are presented in Table. 5.
\begin{table}[h]
\centering
\label{tab:test_data_stats}
\renewcommand{\thetable}{H} 
\begin{tabular}{l|rrr}
\toprule
\textbf{Name} & \textbf{\# Graph pairs} & \textbf{Avg $|V|$} & \textbf{Avg $|E|$} \\ \midrule
AIDS          & 1000                    & 8.8                & 8.8                \\ 
Linux         & 1000                    & 7.6                & 7.0                \\ 
IMDB          & 967                     & 12.2               & 57.0               \\
ogbg-molhiv   & 902                     & 23.0               & 49.5               \\ 
ogbg-molpcba  & 859                     & 25.0               & 54.1               \\ 
ogbg-code2    & 968                     & 36.7               & 35.7               \\ \bottomrule
\end{tabular}
\caption{Test Data Statistics}
\end{table}

\begin{table*}[ht!]
\vspace{-0.1in}
\centering
\caption{Accuracy Comparison (MAE) among top performing baselines (based on avg. rank in Table 3). The top-$3$ lowest MAEs in each dataset are highlighted in various shades of green, with darker shades indicating lower MAE (better).In cases where the MAE becomes the same after rounding to two places after the decimal, an asterisk (*) indicates the lower (better) value when additional decimal places are considered.}
\label{tab:mae}
\setlength{\tabcolsep}{9pt}
\scalebox{0.90}{
\begin{tabular}{@{}p{1.8cm}|l>{\centering\arraybackslash}p{1.7cm}>{\centering\arraybackslash}p{1.7cm}>{\centering\arraybackslash}p{2.0cm}>{\centering\arraybackslash}p{2.0cm}>{\centering\arraybackslash}p{2.2cm}|>{\centering\arraybackslash}p{1.8cm}>{\centering\arraybackslash}p{1.4cm}@{}}
\toprule
\textbf{Methods} & \textbf{AIDS} & \textbf{Linux} & \textbf{IMDB} & \textbf{ogbg-molhiv} & \textbf{ogbg-code2} & \textbf{ogbg-molpcba} & \textbf{Avg. Rank} \\ 
\midrule
 \methodname  & \cellcolor{shade1}0.22 & \cellcolor{shade1}0.01 & \cellcolor{shade2}0.06 & 2.31 & \cellcolor{shade2}3.45 & \cellcolor{shade3}2.63 & 2.5 \\ 
 \methodname-MIX  & \cellcolor{shade2}0.27 & \cellcolor{shade1}0.01* & \cellcolor{shade1}0.04 & \cellcolor{shade3}2.28 & \cellcolor{shade1}3.34 & 2.81 & \cellcolor{yellow}2.17 \\ 

    \greed        & \cellcolor{shade3}0.48 & 0.31 & \cellcolor{shade3}3.00 & \cellcolor{shade2}2.17 & \cellcolor{shade3}4.30 & \cellcolor{shade1}1.77 & 2.83 \\ 
         ERIC        & 0.84 & \cellcolor{shade3}0.19 & 10.19 & 2.35 & 15.19 & \cellcolor{shade2}2.02 & 3.83 \\ 
     \gedgnn       & 0.92 & 0.29 & 4.43 & \cellcolor{shade1}1.75 & 16.68 & 4.58 & 4 \\ 
\bottomrule
\end{tabular}}
\end{table*}
}
\begin{table*}[ht!]
\vspace{0.2in}
\centering
\setlength{\tabcolsep}{9pt}
\scalebox{0.90}{
\begin{tabular}{@{}p{1.8cm}|l>{\centering\arraybackslash}p{2.4cm}>{\centering\arraybackslash}p{2.4cm}>{\centering\arraybackslash}p{2.4cm}>{\centering\arraybackslash}p{2.2cm}>{\centering\arraybackslash}p{2.2cm}>{\centering\arraybackslash}p{2.2cm}@{}}
\toprule
 \textbf{Method} & \textbf{AIDS} & \textbf{Linux} & \textbf{IMDB} & \textbf{ogbg-molhiv} & \textbf{ogbg-code2} & \textbf{ogbg-molpcba} \\ 
\midrule
\methodname  & 0.83 & $\approx1$ & 0.99 & 0.18 & 0.11 & 0.12  \\
\bottomrule
\end{tabular}}
\caption{EMR results of \methodname for all datasets. Values in the range (0.99,1) are denoted as $\approx1$}
\label{tab:EMR_genie}
\end{table*}

\begin{figure}[t]
\vspace{-0.05in}
\centering
    \includegraphics[width=0.9\textwidth]{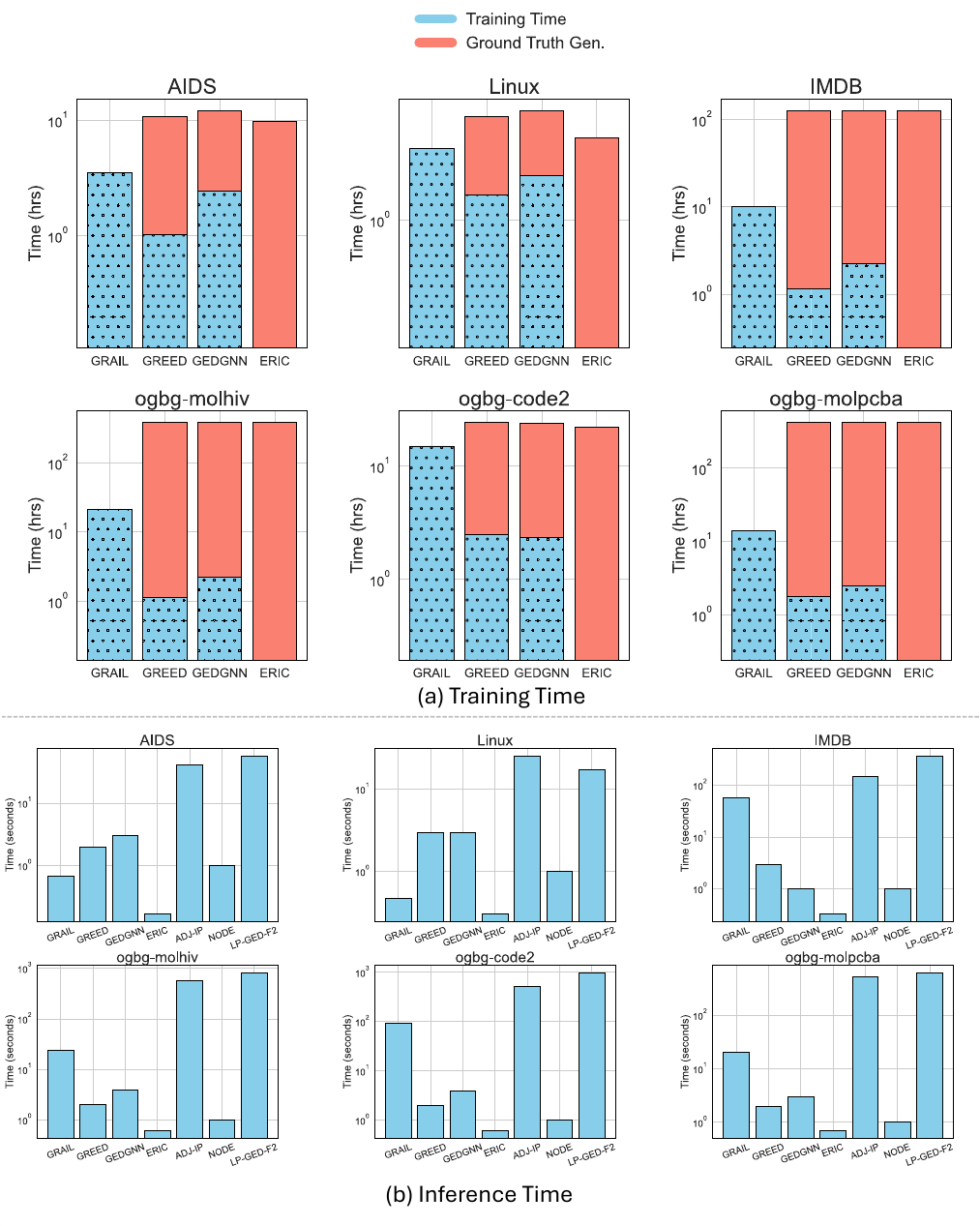}
    \vspace{-0.1in}
    \caption{(a) Training Time: Comparison of \methodname with the top-3 neural methods in Table ~\ref{tab:rmse}. (b) Inference Time: Comparison of \methodname with the top-3 neural and non-neural methods in Table ~\ref{tab:rmse}. The top-3 methods have been selected based on avg. ranks.\\ \textbf{Note:} Ground truth generation time is the same for a dataset (9 hrs 43 min : AIDS, 3 hrs 25 min : Linux, 124 hrs 25 min : IMDB, 379 hrs 19 min : ogbg-molhiv, 21 hrs 42 min : ogbg-code2, 414 hrs 30 min : ogbg-molpcba) for all neural methods, but appears to be different in the plots due to log scale conversion.}
    \label{fig:train_inference_time_comparison}
    \vspace{-0.1in}
\end{figure}

 \begin{figure}
    \includegraphics[width=\linewidth]{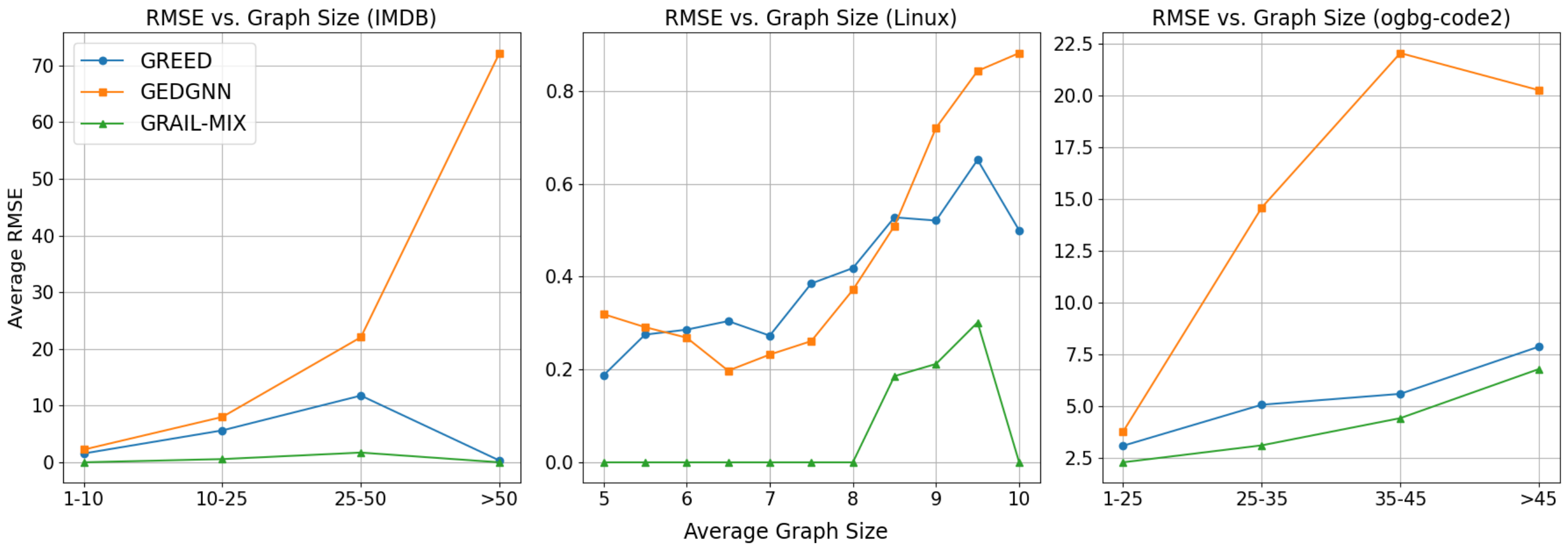}
    \caption{Avg. RMSE vs. Avg. Graph Size comparison on IMDB, Linux and ogbg-code2 datasets. \methodname-\textsc{Mix} outperforms the best baselines at both smaller and larger graph sizes. The rate of increase of error is lower for \methodname-\textsc{Mix} as opposed to \greed and \gedgnn with increasing average graph size.}
    \label{fig:rmse_vs_graph_size}
\end{figure}

\begin{figure}
 \centering
   \hspace{0.5cm}
    \includegraphics[width=0.8\linewidth]{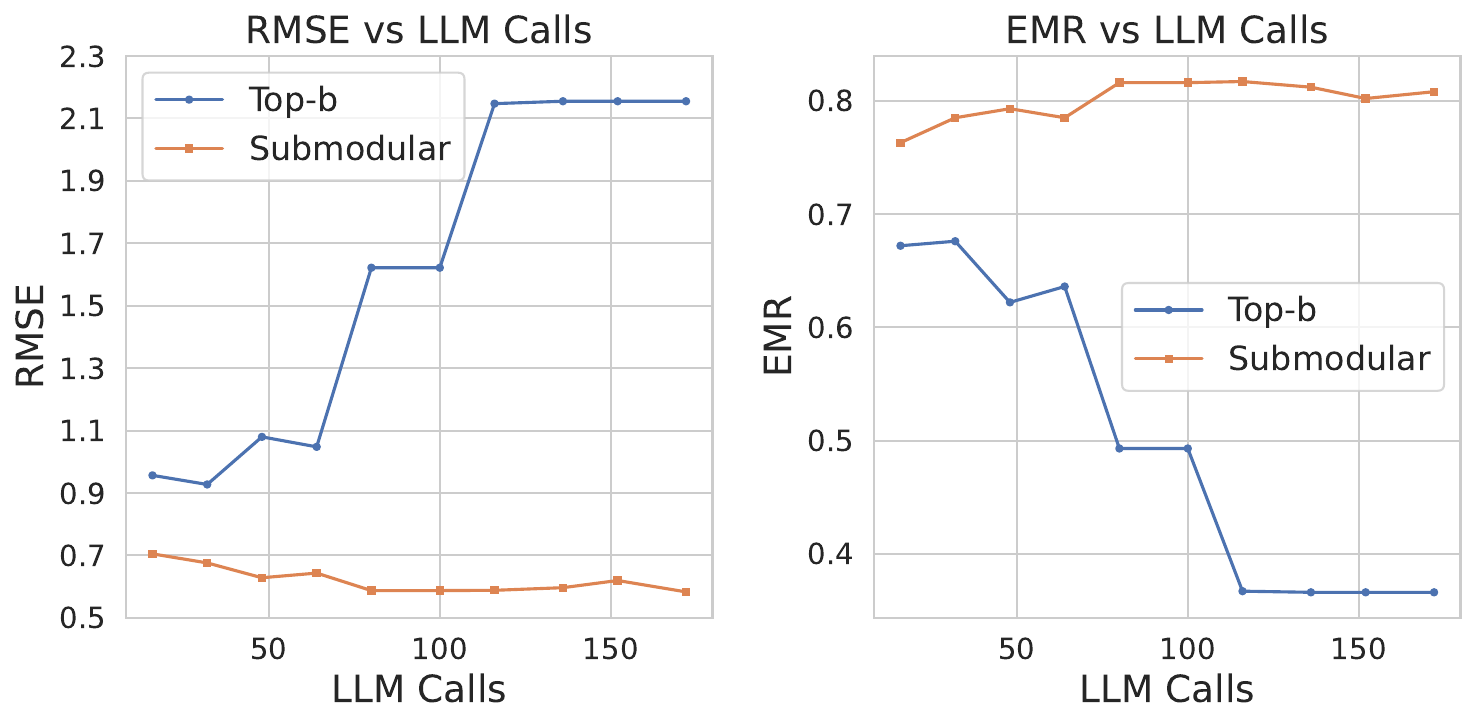}
    \caption{Performance comparison of Top-$b$ vs. Greedy Submodular on the test set of AIDS dataset with an increasing number of LLM calls.}
    \label{fig:topk_vs_greedy_aids}
\end{figure} 

\begin{figure}
 \centering
    \includegraphics[width=1.1\linewidth]{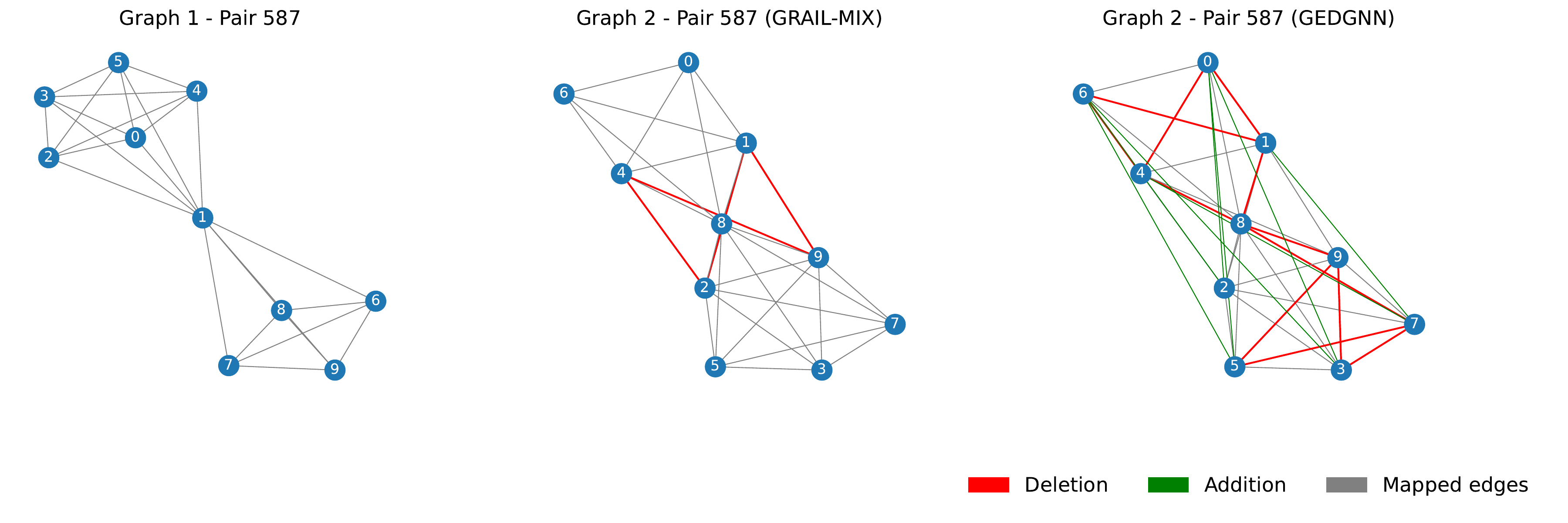}
    \caption{IMDB Case Study: The left-most graph represents Graph 1, while the middle and right-most graphs depict Graph 2 with predicted edits from \gmix (Fig: \ref{fig:casestudy_program}) and \gedgnn, respectively. The red and green edges in each graph indicate the edge edits predicted by both methods. Ground Truth GED:4, \gmix GED:4, \gedgnn Mapping's GED: 20.}
    \label{fig:casestudy_graph}
\end{figure} 

\begin{figure}
 \centering
    \includegraphics[width=1.2\linewidth]{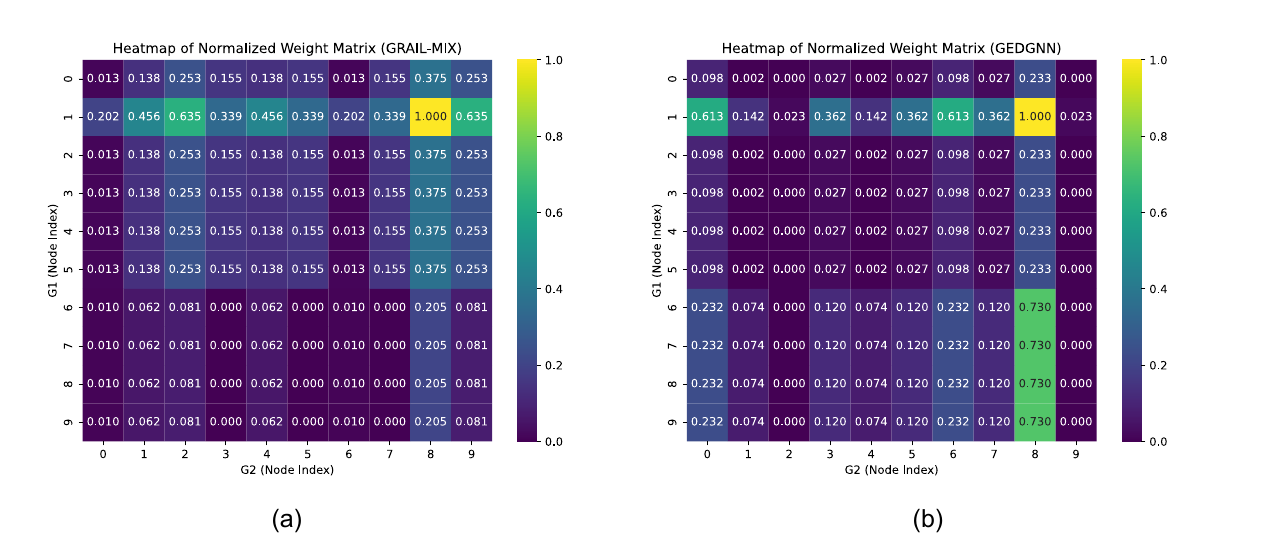}
    \caption{IMDB Case Study: Heatmap of weight matrix generated by (a) \methodname-\textsc{Mix}  (Fig: \ref{fig:casestudy_program}) and (b) \gedgnn}
    \label{fig:casestudy_heatmap}
\end{figure} 

\begin{figure}
 \centering
   \hspace{0.5cm}
    \includegraphics[width=0.98\linewidth]{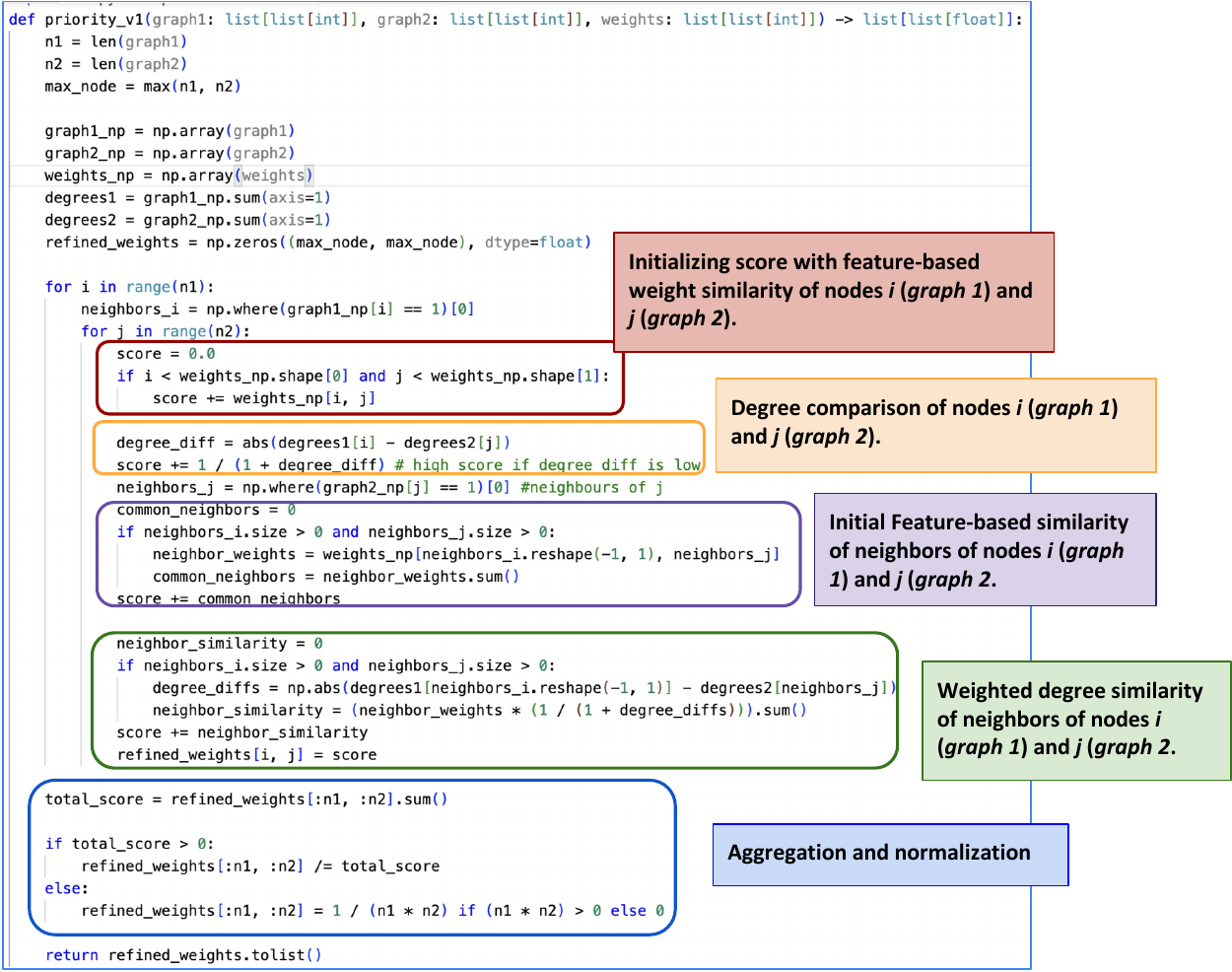}
    \caption{IMDB Case Study: Program discovered by \methodname-\textsc{Mix} that has minimum individual RMSE on IMDB dataset.}
    \label{fig:casestudy_program}
\end{figure}

\subsubsection{Efficiency Analysis}
\label{app:efficiency}
The training and inference time analysis is shown in Fig.~\ref{fig:train_inference_time_comparison}.

\textbf{Training Time:} From Fig.~\ref{fig:train_inference_time_comparison}a, we observe that \methodname is significantly more efficient than the neural baselines. Neural methods require NP-hard ground truth training data, which involves extensive computation times of up to 15 days. Additionally, note that \methodname-\textsc{Mix} requires training only once while performing on par with \methodname and neural baselines in terms of approximation error (see Table~\ref{tab:rmse}). 

\textbf{Inference Time:} We compare the inference time of \methodname with the top three neural and non-neural methods from Table~\ref{tab:rmse}, as shown in Fig.~\ref{fig:train_inference_time_comparison}b. At the onset, we point out that while  \methodname infers on CPUs and provides the node mapping in addition to the predicted \ged, neural baselines rely on GPUs and only provide the \ged.  Hence, neural methods have a lower computational workload while having access to more powerful computational resources. Results indicate that \methodname achieves faster inference times than neural baselines for smaller and sparser datasets, such as AIDS and Linux. However, inference times increase for larger and denser datasets, such as IMDB and ogbg, due to the computational overhead of computing mappings. The maximum recorded inference time is 94.6 seconds for 968 graph pairs in the ogbg-code2 dataset (\(\sim 0.1\) seconds per pair), which remains reasonable considering the various advantages of \methodname, including its independence from ground truth data, one-time training for \methodname-\textsc{Mix}, and strong generalization capabilities. Furthermore,  the efficiency of \methodname's programs can be further improved through human intervention or translation to more efficient languages, such as C.
\comment{\begin{figure*}[ht!]
\centering
\begin{lstlisting}[caption={Example of a priority function generated by \methodname}, label={lst:calculate_ged}, captionpos=b]
def priority(graph1: list[list[int]], graph2: list[list[int]], weights: list[list[int]]) -> list[list[float]]:
  """Improved version of `priority_v0`.
  This version uses a normalized degree similarity and handles cases where one or both nodes have no neighbors more gracefully.
  """
  n1 = len(graph1)
  n2 = len(graph2)
  max_node = max(n1, n2)
  refined_weights = [([0.0] * max_node) for _ in range(max_node)]
  for i in range(n1):
    for j in range(n2):
      node_similarity = weights[i][j]
      degree_i = sum(graph1[i])
      degree_j = sum(graph2[j])

      # Normalized degree similarity
      max_degree = max(1, degree_i, degree_j) # Avoid division by zero
      degree_similarity = 1 - (abs(degree_i - degree_j) / max_degree)


      neighbors_i = [k for k in range(n1) if graph1[i][k]]
      neighbors_j = [l for l in range(n2) if graph2[j][l]]
      neighbor_similarity = 0
      common_neighbors = 0
      for ni in neighbors_i:
        for nj in neighbors_j:
          neighbor_similarity += weights[ni][nj]
          common_neighbors += 1

      if common_neighbors > 0:
        neighbor_similarity /= common_neighbors
      elif (degree_i > 0) and (degree_j > 0):  # Both have neighbors but no common ones
        neighbor_similarity = 0.0           # Explicitly set to 0 for dissimilarity
      else:
          neighbor_similarity = 1.0 if (degree_i == 0 and degree_j == 0) else 0.0 # 1 if both have no neighbors, else 0


      # Improved neighbor similarity factor, avoids potential division by zero
      neighbor_similarity_factor = neighbor_similarity * (min(degree_i, degree_j) / max(1, max(degree_i, degree_j)))


      refined_weights[i][j] = ((2 * node_similarity * (1 + degree_similarity) + neighbor_similarity_factor) / (3 + degree_similarity))
  return refined_weights

\end{lstlisting}
\end{figure*}}


\end{document}